\documentclass[10pt,twocolumn,letterpaper]{article}

\usepackage{cvpr}              
\definecolor{cvprblue}{rgb}{0.21,0.49,0.74}
\usepackage[pagebackref,breaklinks,colorlinks,allcolors=cvprblue]{hyperref}
\usepackage{algorithmic}
\usepackage{listings}%
\usepackage{algorithm}%
\usepackage{amsthm}
\usepackage[accsupp]{axessibility}
\usepackage{amsmath}
\usepackage{booktabs} 
\usepackage{multirow}
\usepackage{makecell}
\usepackage{adjustbox}
\usepackage[most]{tcolorbox}
\usepackage{lipsum} 
\usepackage{xcolor}
\usepackage[table]{xcolor}
\newtheorem{definition}{Definition}
\newtheorem{theorem}{Theorem}

\newtheorem{corollary}{Corollary}
\newtheorem*{theorem*}{Theorem}
\newtheorem*{corollary*}{Corollary}
\definecolor{cbg}{HTML}{F7F3ED}    
\definecolor{cframe}{HTML}{6B5341}  
\definecolor{ctitle}{HTML}{2E2219}  
\definecolor{ctitlebar}{HTML}{EDE3D8}
\definecolor{lavender1}{RGB}{230,230,250}
\definecolor{lavender2}{RGB}{210,210,240}
\def\paperID{***} 
\def\confName{CVPR}
\def\confYear{2026}

\title{Beyond Heuristic Prompting: A Concept-Guided Bayesian Framework for Zero-Shot Image Recognition}

\author{
Hui Liu\textsuperscript{1},
Kecheng Chen\textsuperscript{1},
Jialiang Wang\textsuperscript{1,2},
Xianming Liu\textsuperscript{2},
Wenya Wang\textsuperscript{3},
Haoliang Li\textsuperscript{1}\\[0.6em]
\textsuperscript{1}City University of Hong Kong \quad
\textsuperscript{2}Harbin Institute of Technology \quad
\textsuperscript{3}Nanyang Technological University
}

\begin{document}
\maketitle
\begin{abstract}
Vision-Language Models (VLMs), such as CLIP, have significantly advanced zero-shot image recognition. However, their performance remains limited by suboptimal prompt engineering and poor adaptability to target classes. While recent methods attempt to improve prompts through diverse class descriptions, they often rely on heuristic designs, lack versatility, and are vulnerable to outlier prompts. This paper enhances prompt by incorporating class-specific concepts. By treating concepts as latent variables, we rethink zero-shot image classification from a Bayesian perspective, casting prediction as marginalization over the concept space, where each concept is weighted by a prior and a test-image conditioned likelihood. This formulation underscores the importance of both a well-structured concept proposal distribution and the refinement of concept priors. To construct an expressive and efficient proposal distribution, we introduce a multi-stage concept synthesis pipeline driven by LLMs to generate discriminative and compositional concepts, followed by a Determinantal Point Process to enforce diversity. To mitigate the influence of outlier concepts, we propose a training-free, adaptive soft-trim likelihood, which attenuates their impact in a single forward pass. We further provide robustness guarantees and derive multi-class excess risk bounds for our framework. Extensive experiments demonstrate that our method consistently outperforms state-of-the-art approaches, validating its effectiveness  in zero-shot image classification. Our code is available at \url{https://github.com/less-and-less-bugs/CGBC}.
\end{abstract}

\section{Introduction}
\label{sec:intro}
Zero-shot image recognition refers to classifying images into  classes unseen during training. Vision-language models\footnote{Since CLIP is trained on hundreds of millions of text-image pairs sourced from the internet, ~\citet{clip} evaluate its zero-shot capability at the dataset level.} (VLMs)~\citep{clip, blip2, evaclip, llava, xiao2025exploring, xiao2024vanessa, yin2026aosnet} have revolutionized this area through large-scale natural language supervision, enabling the understanding of  a broad range of visual concepts far beyond traditional model capabilities.  A notable example is CLIP~\citep{clip}, which achieves zero-shot classification without task-specific fine-tuning by aligning image and text representations in a shared embedding space. Specifically, the test images are projected using the visual encoder, while class names are incorporated into a predefined prompt template (e.g., ``A photo of \{class\}") and encoded via the text encoder. Classification is determined by maximal similarity between these image and text embeddings.

However, the performance of such paradigm often remains suboptimal in real-world scenarios, likely due to ineffective prompt engineering~\citep{coop, cocoop} and limited adaptation to target classes~\citep{clipadapter, tipadapter, yu2023task, huang2025cosmic}. To address this, existing approaches~\citep{shu2022tpt, ma2024swapprompt, feng2023diverse, matcvpr24, lafon2025cliptta} apply image augmentations to create multiple views and aggregate predictions via specific optimization objectives such as entropy minimization. Another line of work~\citep{pratt2023does, zeroshotcvpr23, novack2023chilszeroshot} enhances the single prompt with multiple LLM-generated class descriptions or subclass labels (e.g., “bulldog” for “dog”), to better exploit the model's knowledge beyond class names. Notably, the latter group only requires  offline prompt generation and encoding, thus eliminating the computational latency of test-time view processing and optimization.

Despite their promise, prompt-enhanced methods exhibit limited versatility, particularly in fine-grained classification tasks (e.g., “2000 AM General Hummer SUV” in the Cars dataset~\citep{krause20133d}), where defining meaningful sub-classes is inherently difficult. Moreover, these approaches rely on heuristic designs without theoretical grounding and systematic evaluation. For example, CuPL~\citep{parkhi2012cats} queries LLMs with multiple class-related questions (e.g., What does \{class\} look like?), uses the responses as prompts, and averages their similarity to the test image for classification, without further analysis of which types of prompts are most effective. Additionally, as shown in Fig. \ref{fig:fig1}, similarity scores between these enhanced prompts and test images often follow a skewed or long-tail distribution, indicating the presence of outlier prompts that may degrade classification accuracy.
\begin{figure}[htbp]
    \centering
    \includegraphics[width=0.9\linewidth]{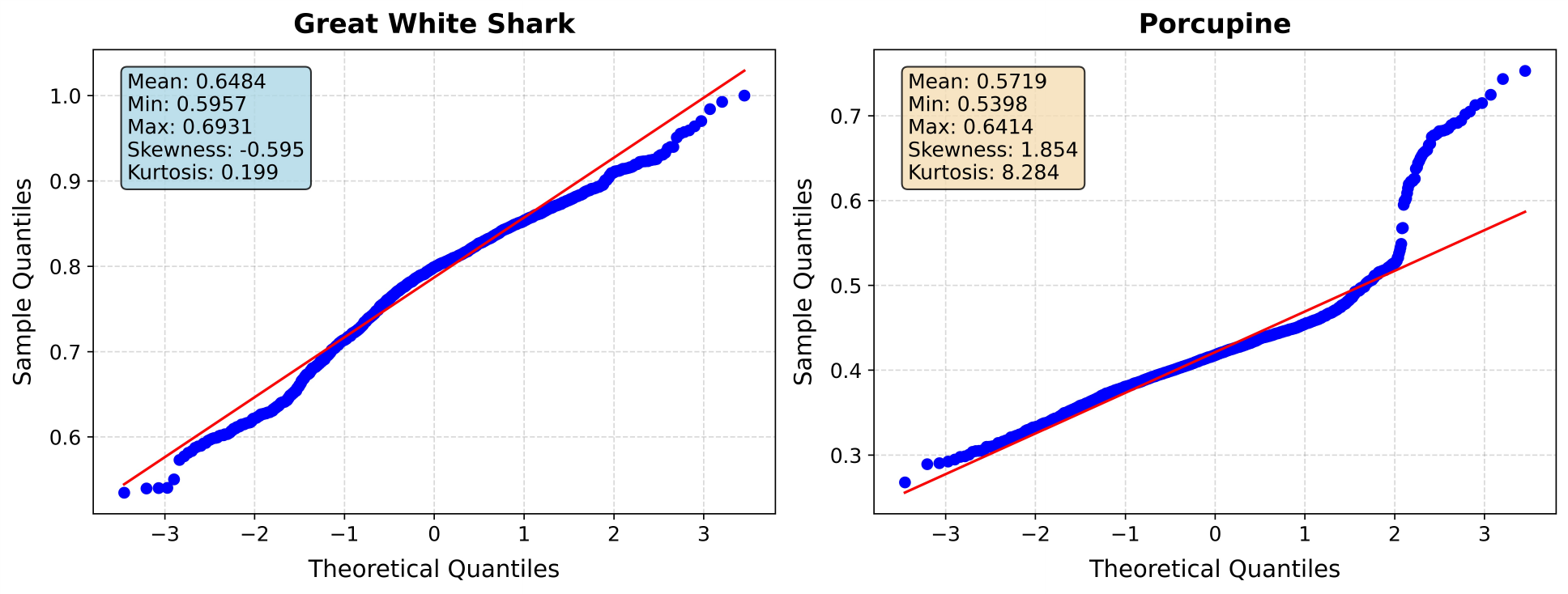}
    \caption{Q--Q plots of similarity distributions between all  Great White Shark images in ImageNet dataset and CuPL~\citep{pratt2023does} prompts for this class and Porcupine. Departures from the normal reference line indicate non-normality. Pronounced skewness is taken as $|\text{skewness}|>0.5$, and heavy tails as excess kurtosis $>3$.}
    \label{fig:fig1}  
    \vspace{-15pt}
\end{figure}

To address these challenges, we enhance prompts with class-specific concepts to leverage the rich visual knowledge encoded in VLMs, making the approach adaptable to varying levels of classification granularity. By treating concepts as latent variables, we rethink zero-shot image classification from  a Bayesian perspective and formulate it as marginalization over the concept space, where each concept is weighted by a prior and a likelihood conditioned on the input image. This formulation yields two key insights: 1) the need for a well-structured concept proposal distribution (sampling algorithm) to approximate the vast concept space, and 2) the importance of likelihood-based refinement of the prior, highlighting the limitations of simple heuristic averaging across prompts~\citep{pratt2023does}. To instantiate this Concept-Guided Bayesian Classification (CGBC) framework, building on classical concept discovery principles~\citep{glanois2022neuro}, we suggest the concept proposal distribution should satisfy discriminability (inter-class distinction), compositionality (combination of atomic concepts), and diversity (minimal semantic redundancy), which collectively optimize concept effectiveness under constrained budgets. As such, we propose a LLM-driven multi-stage concept synthesis pipeline. The LLM is iteratively prompted to generate class-specific atomic concepts that emphasize inter-class distinctions. We then sample combinations of these atomic concepts and select a subset using Determinantal Point Process (DPP)~\citep{kulesza2012determinantal} to reduce semantic overlap in the selected concept set. To mitigate the impact of outlier concepts, we introduce an adaptive soft-trim based likelihood function that provides robust mean estimation through outlier down-weighting in a single forward pass. We also establish theoretical robust guarantees and multi-class excess risk bounds for our framework. Extensive evaluation across eleven image recognition tasks demonstrates CGBC's consistent superiority over existing state-of-the-art methods.
The contribution of our work is summarized as follows: 
\begin{itemize}
    \item We rethink VLM-based zero-shot image recognition from a Bayesian perspective and  underscore the importance of a well-structured concept proposal distribution and input-conditioned refinement through a likelihood function.
    
    \item We propose a multi-stage concept synthesis pipeline to construct a more expressive and sampling-effective concept proposal distribution  and introduce a training-free adaptive soft-trim likelihood function to mitigate the negative influence of outlier concepts. We further provide theoretical guarantees for our framework.
    
    \item We empirically validate the effectiveness of our framework across a broad range of zero-shot image recognition tasks, demonstrating the practical benefits of our concept  synthesis strategy and likelihood design.
\end{itemize}

\section{Related Works}
\label{sec:related}
\subsection{Zero-shot Image Recognition in VLMs} 
The generalization capabilities of vision-language models have attracted significant attention in recent years due to the challenges of fine-tuning all model parameters while preserving robust performance across domains~\citep{coop, lester2021power}. While most existing research concentrates on few-shot settings~\citep{cocoop, zhu2023prompt, chen2022plot, huang2022unsupervised, wang2023improving, tipadapter, shang2024incremental, martin2024transductive, zanellaboostingnips24, silva2024closer, udandarao2023sus, zhang2023prompt}, our work focuses on the more challenging zero-shot recognition scenario, where no labeled examples are available for any class during testing. Current zero-shot approaches can be broadly categorized into two groups. The first leverages test-time data augmentation, generating diverse views of input images to promote prediction consistency, either through prompt tuning~\citep{shu2022tpt, ma2024swapprompt, feng2023diverse} or view-specific optimization~\citep{matcvpr24}. However, these methods incur significant computational overhead due to the need for test-time view encoding and optimization. The second category enhances prompts through using LLMs to generate detailed descriptions~\citep{pratt2023does, COLA} or constructing hierarchical structures (e.g., subclasses of the target class)~\citep{zeroshotcvpr23, novack2023chilszeroshot}. While such methods reduce test-time latency through offline prompt generation and encoding, they often exhibit limited adaptability to fine-grained recognition, where defining meaningful subclasses is a non-trivial task. Moreover, they typically rely on heuristic prompt design and lack a principled framework for both prompt construction and subsequent inference. As such, we introduce class-specific concepts as latent variables and provide a theoretically grounded view by rethinking zero-shot recognition from a Bayesian perspective.

\subsection{Concept Discovery} 
Concept discovery, the inference of general patterns from specific examples, is fundamental across cognitive science and inductive logic programming~\citep{ellis2021dreamcoder, liu2023unsupervised, chater2008probabilistic}. A common approach involves defining a hypothesis space of potential concepts and applying Bayesian inference to determine which hypothesis most likely generates the observed data probabilistically~\citep{ellis2021dreamcoder, saad2019bayesian, tian2020learning}. This paradigm raises a critical issue of balancing the expressiveness of the hypothesis space with computational tractability. Prior work~\citep {qu2020rnnlogic, glanois2022neuro} addresses this by promoting sparse concept representations while leveraging compositional structures to enhance expressive power. Task-specific objectives further refine discovery by aligning it with downstream utility~\citep{liu2023interpretable, ellis2024human}. Building on these insights, we argue that an effective concept proposal distribution should satisfy three criteria, including discriminability (to distinguish between classes), compositionality (to support meaningful combinations of atomic concepts), and diversity (to minimize semantic redundancy).

\section{Concept-Guided Bayesian Classification}
\label{sec:preliminary}
In a standard zero-shot image recognition setting, consider a set of candidate classes $\mathcal{Y} = \{ Y_1, \dots, Y_K \}$, where $Y_i$ denotes the $i$-th class and $K$ the total number of classes. Given a test image $X$, CLIP-like VLMs represent each class $Y_i$ using the prompt “A photo of a $Y_i$”.  The model then processes both the input image $X$ and the textual description through respective encoders, yielding $l_2$-normalized visual and textual embeddings.  The posterior probability $p(Y_i|X)$ is computed via the cosine similarity between these embeddings, normalized to the interval $[0,1]$. The classification decision can be formulated as:
\begin{equation}
\label{eq:1}
    \hat{Y} = \arg\max_{Y_i \in \mathcal{Y}} \; p(Y_i | X).
\end{equation}
\begin{figure*}[htbp]
    \centering
    \includegraphics[width=0.95\linewidth]{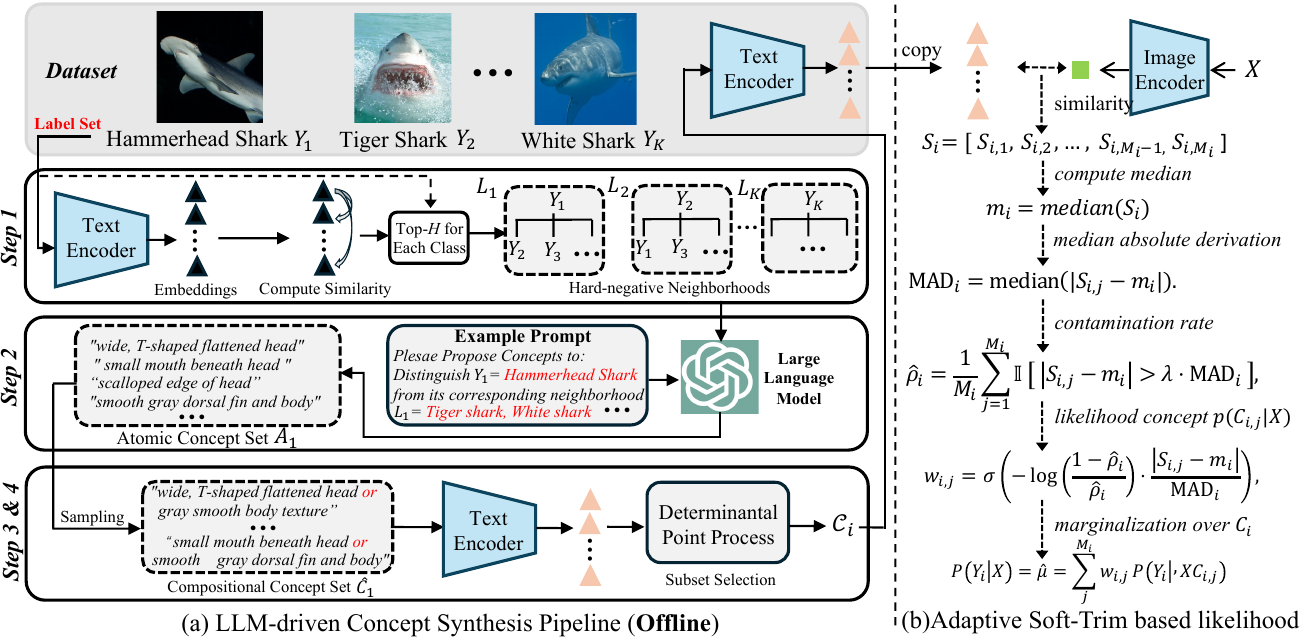}
    \caption{Instantiation of our proposed Concept-Guided Bayesian Classification (CGBC) framework. (a) LLM-driven Concept Synthesis Pipeline: Constructs class-aware hard-negative neighborhoods in Step 1, generates discriminative atomic concepts via LLM prompting in Step 2, composes higher-order concepts through random sampling in Step 3, and selects diverse concepts using Determinantal Point Processes (DPPs) in Step 4. (b) Adaptive Soft-Trim Likelihood Estimation: Estimates contamination rates in similarity distributions and downweights outlier concepts that deviate significantly from the median, improving robustness in concept-based image classification. }
    \label{fig:model}  
    \vspace{-15pt}
\end{figure*}

While this zero-shot paradigm demonstrates empirical efficacy, it solely leverages the encoded knowledge within VLMs' parameters, neglecting critical prior information such as conceptual attributes and inter-class relationships beyond class labels. Therefore, we propose to integrate class-specific concepts into the original prompt template to generate more discriminative descriptions of the form ``A photo of a \{class\} with \{concept\}". Through hypothesizing a test image $X$ can be generated through a set of latent concepts $\mathcal{C} = \{ C_1, \ldots, C_K \}$, where each $C_i$ represents an infinite concept space corresponding to class $Y_i$, we rethink this task from a Bayesian perspective. The posterior probability $p(Y_i| X)$ can be expressed via marginalization:
\begin{equation}
    p(Y_i|X) = \sum_{C_i} p(Y_i|X, C_i) p(C_i|X),
\end{equation}where $p(C_i|X)$ denotes the posterior distribution over latent concept space $C_i$ conditioned on image $X$ and $p(Y_i| X, C_i)$ can be computed analogously to the standard zero-shot setting. However, the direct computation of $p(C_i|X)$ is intractable due to the infinite-dimensional nature of the concept space. As such, we apply Bayes' theorem:
\begin{equation}
    p(C_i|X) = \frac{p(X|C_i) p(C_i)}{p(X)}\propto p(X|C_i) p(C_i), 
\end{equation}where $p(C_i)$ represents the prior distribution over $C_i$, encoding rich world knowledge for image recognition \citep{kahneman2011thinking}. The likelihood term $p(X|C_i)$ quantifies the compatibility between image $X$ and concept $C_i$, refining the coarse prior into a more precise posterior. The normalization term $p(X)$ is constant and thus omitted from relative probability calculations. The final classification decision is obtained through marginalization over $C_i$:
\begin{equation}
\label{eq:3}
    p(Y_i|X) \approx \sum_{C_i} p(Y_i|X, C_i)p(X|C_i) p(C_i).
\end{equation}
To address the computational challenges posed by infinite concept space, we introduce a proposal distribution $q(C_i)$ inspired by sampling theory \citep{kloek1978bayesian}, which enables the approximation of the summation in Eq.~\eqref{eq:3} through a finite sum over sampled concepts $\mathcal{C}_i=\{C_{i,1}, \ldots, C_{i,M_i}\}$ for each class where each concept will be transformed into a prompt:
\begin{equation}
\label{eq:6}
p(Y_i|X) \approx \sum_{{C_{i,j}\in \mathcal{C}_i}}^{M_i} p(Y_i|X, C_{i,j})p(X|C_{i,j}).
\end{equation} 
In summary, this section presents the concept-guided Bayesian classification framework (CGBC) and, more importantly, underscores the importance of constructing  appropriate concept proposal distributions and test image-conditioned likelihood for prior refinement in VLM-based zero-shot image recognition.
\section{Framework Implementation}
\label{sec:method}
In this section, we present the implementation of our proposed Concept-Guided Bayesian Classification (CGBC) framework. As illustrated in Fig. \ref{fig:model}, we introduce a LLM-driven, multi-stage concept synthesis pipeline designed to construct a concept proposal distribution characterized by discriminability, compositionality, and diversity. Additionally, we incorporate an Adaptive Soft-Trim based Likelihood function to mitigate the negative impact of outlier concepts, refining the concept prior based on input images. The details of these two components and theoretical analysis are elaborated in  Sec.\ref{sec:4.1}, Sec. \ref{sec:4.2} and Sec. \ref{sec:4.3}, respectively.
\subsection{LLM-driven Concept Synthesis Pipeline}
\label{sec:4.1}
Defining an expressive yet computationally tractable concept proposal distribution $q(C_i)$ is inherently challenging due to the effectively infinite nature of the concept space. Inspired by prior work in classical concept discovery~\citep{qu2020rnnlogic, glanois2022neuro, ellis2021dreamcoder}, we argue that an effective $q(C_i)$ should satisfy  the below three properties: 
\begin{definition}[Discriminability] The concept from $q(C_i)$ should accurately characterize the target class $Y_i$ while enhancing inter-class separability. Ideally, $q(C_i)$ should be conditioned on both $Y_i$ and its complement $\mathcal{Y}_{/i}$ (i.e., all other classes). For example, when distinguishing Hammerhead Shark from other shark species, the concept ``T-shaped flattened head'' is highly discriminative, whereas a generic concept like ``smooth gray body'' is not.
\end{definition}
\begin{definition}[Compositionality] Distinct atomic concepts should be combined to form higher-order composite concepts (e.g., ``T-shaped flattened head or smooth gray body''). This compositional structure enables richer semantics and robustness, especially when individual atomic concepts may be insufficient or ineffective.
\end{definition}
\begin{definition}[Diversity]The concepts from $q(C_i)$ should exhibit low semantic redundancy to maximize coverage of the true concept prior $p(C_i)$ under a fixed sampling budget. Since near-duplicate concepts offer diminishing returns and incur extra computational cost, promoting diversity improves both expressiveness and computational efficiency.
\end{definition}

To simultaneously achieve these three aspects, we propose a four-stage concept synthesis pipeline, which integrates the extensive world knowledge of LLMs with the representational strength of VLMs as follows:

\noindent\textbf{Step 1}. Construction of Class-aware Hard-negative Neighborhoods: Given a set of classes $\mathcal{Y} = \{Y_1, \dots, Y_K\}$, we first encode each class label using CLIP's text encoder to obtain normalized embeddings and compute pairwise cosine similarities between these embeddings. For each class $Y_i$,  we identify the $H$ most semantically similar classes, forming its hard-negative neighborhood $L_i$, which provides effective approximation to $\mathcal{Y}{/i}$ and 
establishes the foundation for ensuring discriminability in subsequent procedures.

\noindent\textbf{Step 2}. Atomic Concept Generation via Contrastive Prompting: Leveraging the class-specific hard-negative neighborhood $L_i$, we design contrastive prompts that instruct LLMs to generate candidate concepts  that distinguish $Y_i$ from  $L_i$. This procedure can produce a set of atomic concepts $\mathcal{A}_i = \{A_{i,1}, \dots, A_{i,M_{A}}\}$ after a few API calls for $Y_i$, where $M_{A}$ denotes the number of atomic concepts. Unlike conventional prompting methods that treat class labels in isolation~\citep{novack2023chilszeroshot}, our approach encourages the generation of more discriminative concepts by explicitly modeling inter-class contrast. To reduce redundancy, we compute cosine similarity between newly generated and existing concepts, pruning any candidate with a similarity exceeding 0.9.

\noindent\textbf{Step 3}. Compositional Concept Construction: For each class $Y_i$, we sample $N_c$  combinations of atomic concepts (with a fixed number of concepts per prompt) from the pool $\mathcal{A}_i$ without replacement. These combinations are composed using logical operators (e.g., “or”) to generate a candidate set of composite concepts, denoted by $\hat{C}_i$. This compositional strategy enhances the expressiveness and robustness of the concept space compared to using atomic concepts alone, while adhering to the same prompt budget.

\noindent\textbf{Step 4}.  Subset Selection via Determinantal Point Processes: To obtain a diverse and representative subset of concepts for class $Y_i$, each candidate composite concept in $\hat{C}_i$ is encoded using CLIP’s text encoder to produce normalized embeddings. These embeddings are used to construct a similarity kernel $K(\hat{C}_{i,j}, \hat{C}_{i,u}) = \phi(\hat{C}_{i,j})^\top \phi(\hat{C}_{i,u})$, where $\phi(\cdot)$ denotes the CLIP embedding function. The DPP~\cite{kulesza2012determinantal} is then applied to select a subset $\mathcal{C}_i \subseteq \hat{C}_i$ of size $M_i$ that efficiently minimizes semantic redundancy among concepts under a constrained budget.

Through our multi-stage, LLM-driven concept synthesis pipeline, we construct class-specific concept sets $\{\mathcal{C}_1, \dots, \mathcal{C}_K\}$, each defining a proposal distribution $q(C_i)$ for class $Y_i \in \mathcal{Y}$. By integrating discriminability via contrastive prompting, compositionality through atomic concept combination, and diversity using DPP-based subset selection, the resulting $q(C_i)$ provides an effective approximation of the actual concept prior $p(C_i)$.
\subsection{Adaptive Soft-Trim based Likelihood}
\label{sec:4.2}
Let $\mathcal{S}_i = \{S_{i,1}, \ldots, S_{i,M_i}\}$ denote the set of predicted similarities between a test image $X$ and prompts augmented with the concept set $\mathcal{C}_i$, corresponding to class $Y_i \in \mathcal{Y}$, where each $S_{i,j}$ approximates the conditional probability $p(Y_i| X, C_{i,j})$. Empirically, $\mathcal{S}_i$ often follows a long-tailed distribution, indicating the presence of outlier concepts. To enhance robustness in estimating $p(Y_i | X)$, we introduce an Adaptive Soft-Trim Likelihood, which down-weights concepts whose similarity scores deviate significantly from the median of $\mathcal{S}_i$, thereby mitigating outlier effects.

Since $p(Y_i | X)$ is a scalar, we focus on one-dimensional robust statistics. To formally characterize the presence of outliers, we adopt the univariate Huber Contamination Model \cite{huber1992robust} and the goodness condition  adapted  from \cite{diakonikolas2018robustly,diakonikolas2023near}.
\begin{definition}[Contamination Model] For each class $Y_i$, let $\rho_i \in [0, 0.5)$ denote the unknown concept noise rate. Let $S_{i,j}$ with mean $\mu_i$ and variance terms $\varepsilon_{i,j}$ be i.i.d.  sub-Gaussian random variables with variance proxy $\sigma^2_i$ and $\mathbb{E}[\varepsilon_{i,j}] = 0$. Let $\zeta_{i,j}$ represent arbitrary adversarial contamination. Then:
\begin{equation}
\small
s_{i,j} = \begin{cases}
\mu_i + \varepsilon_{i,j}, & \text{with probability } 1-\rho_i \\
\zeta_{i,j}, & \text{with probability } \rho_i.
\end{cases}
\end{equation}
\end{definition}
\noindent In more accessible terms, $\rho_i$ represents the fraction of outlier concepts in $\mathcal{C}_i$, while $1 - \rho_i$ corresponds to clean (uncontaminated) concepts. The parameter $\mu_i$ serves as the ideal estimator of $p(Y_i| X)$ in the presence of contamination.
\begin{definition}[{$(\sigma_i,\alpha)$-Goodness}] Let similarity scores $\mathcal S_i=\{S_{i,1},\ldots,S_{i,M_i}\}$ have clean mean $\mu_i$ and variance proxy $\sigma_i^2$. We say $\mathcal S_i$ is $(\sigma_i,\alpha)$-good w.r.t.\ $\mu_i$ if:

\begin{enumerate}[leftmargin=*,label=(\arabic*),topsep=0pt,itemsep=2pt]
  \item \textbf{(Median)}  
        There exists an absolute constant $C_1>0$ such that $    \Pr_{S_{i,j}\sim \mathcal{S}_i}\!\bigl[\,
        \bigl|S_{i,j}-\mu_i\bigr|\;\ge\; C_1\,\sigma_i\,\rho_i
    \bigr]\; <\; \tfrac12$.
  \item \textbf{(Mean / Variance under any weight vector)}  
       For any weight vector $\mathcal{W}_i = \{w_{i,1}, \dots, w_{i,M_i}\}$ with $w_{i,j} \in [0,1]$ satisfying the average mass constraint $\frac{1}{M_i} \sum_{j=1}^{M_i} w_{i,j} \ge 1 - \alpha$, define: the weighted mean  $ \hat{\mu}_{i}:=
            \frac{\sum_{j=1}^{M_i} w_{i,j}\,S_{i,j}}
                 {\sum_{j=1}^{M_i} w_{i,j}}$ and variance $  \hat{\sigma}_{i}^{2}:=
            \frac{\sum_{j=1}^{M_i} w_{i,j}\,(S_{i,j}-\hat{\mu}_{i})^{2}}
                 {\sum_{j=1}^{M_i} w_{i,j}}$.
        Then there exist universal constants $C_2,C_3>0$ such that 
        \begin{small}
        \begin{align}
           \tag{2.a}\label{eq:good-mean}
           | \hat{\mu}_{i}-\mu_i|
           &\;\le\; C_2\,\sigma_i\,\alpha\sqrt{\log\!\frac1\alpha}, \\[4pt]
           \tag{2.b}\label{eq:good-var}
           |\hat{\sigma}_{i}^{2}-\sigma_i^{2}|
           &\;\le\; C_3\,\sigma_i^{2}\,\alpha\log\!\frac1\alpha .
        \end{align}
        \end{small}
\end{enumerate}
\end{definition}
\noindent These conditions ensure that the sample median is a robust estimator of the true mean, and that weighted estimates of mean and variance remain accurate even when a fraction $\alpha$ of the data is down-weighted or removed.

Building on the above, we first compute the median $m_i$ of $\mathcal{S}_i$ as an initial robust estimate of $\mu_i$ ($p(Y_i|X)$). To quantify the variability of similarity scores around $m_i$, we compute the Median Absolute Deviation (MAD):
\begin{equation}
\small
\text{MAD}_i = \text{median}\big(|S_{i,j} - m_i|\big).
\end{equation}
We estimate the contamination rate $\hat{\rho}_i$ as the fraction of samples that deviate significantly from the median:
\begin{equation}
\small
  \hat{\rho}_i = \frac{1}{M_i} \sum_{j=1}^{M_i} \mathbb{I}\left[\, \left| S_{i,j} - m_i \right| > \lambda \cdot \text{MAD}_i \,\right],
\end{equation}
where $\mathbb{I}$ denotes the indicator function that returns 1 if the condition inside is true and 0 otherwise. The parameter $\lambda$ is a threshold controlling outlier detection sensitivity, i.e., larger values make the criterion stricter, resulting in more samples being considered outliers. Through interpreting $\hat{\rho}_i$ as the prior that a concept $C_{i,j}$ is an outlier and the normalized deviation $\frac{|S_{i,j} - m_i|}{\text{MAD}_i}$ as observed evidence (i.e., how far a sample deviates from the center), the posterior  that $C_{i,j}$ is clean (non-outlier) is modeled using a logistic form. Specifically, the weight $w_{i,j}$ assigned to concept  $C_{i,j}$ is given by:
\begin{equation}
\label{wij}
\small
w_{i,j} = \sigma\left( -\log\left( \frac{1 - \hat{\rho}_i}{\hat{\rho}_i} \right) \cdot \frac{|S_{i,j} - m_i|\cdot k}{\text{MAD}_i\cdot} \right),
\end{equation}
where $\sigma(\cdot)$ denotes the sigmoid function and $k$ represents its slope. The resulting weight $w_{i,j}$ represents a soft estimation of concept reliability, thus serving as  likelihood $p(C_{i,j}|X)$ to mitigate the negative effect of outlier concept. We then compute a robust estimate of $p(Y_i|X)$ via adaptive marginalization over $C_i$:
\begin{equation}
\small
      P(Y_i|X)=\hat{\mu}=\sum_{j}^{M_i}\frac{w_{i,j}}{\sum_{k=1}^{M_i}}P(Y_i|X,C_{i,j}).
\end{equation}

Finally, following  Eq. \eqref{eq:1} in the standard zero-shot image classification, we  predict the label $\hat{Y}$ by selecting the class with the highest estimated likelihood.
\subsection{Theoretical Analysis}
\label{sec:4.3}
\begin{theorem}[Robust Guarantee]\label{thm:robust_mean}
If the size of concept set $\mathcal{C}_i$ for each class $Y_i$ satisfies $M_i \geq \frac{C_0 \log(1/\delta)}{\rho_i^2}$
for a universal constant $C_0$, then with probability at least $1-\delta$, the estimation error of $\hat{\mu}_{i}$ is bounded by:
\begin{equation}
\small
|\hat{\mu}_{i} - \mu_i| \leq C_1\sigma_i\rho_i + C_2\sigma_i\sqrt{\frac{\log(1/\delta)}{M_i}} + \frac{C_3\sigma_i}{k}
\end{equation}
where $C_1$, $C_2$, and $C_3$ are universal constants and $k$ is sigmoid slope parameter in Eq. \eqref{wij}.
\end{theorem}
\begin{corollary}[Multi-class excess risk]\label{cor:excess_risk}
Under the condition in Theorem~\ref{thm:robust_mean}, the classifier for image $X$ in Eq. \eqref{eq:1} satisfies:
\begin{equation}
\small
\mathcal{R} - \mathcal{R}^* \leq \Pr\left[\text{margin}_{\mu}(X) \leq 2 \cdot \max_i |\hat{\mu}_{i}(X) - \mu_i(X)|\right]
\end{equation}
where $\mathcal{R}$ is the risk of our classifier, $\mathcal{R}^*$ is the optimal Bayes risk, and $\text{margin}_{\mu}(X)$ is the margin of the true mean $\mu_i(X)$. Consequently, the excess risk is bounded by:
\begin{small}
\begin{align}
\mathcal{R} - \mathcal{R}^* 
&\leq \Pr\Big[\text{margin}_{\mu}(X) \leq 2\big(C_1\sigma_{\max}\rho_{\max} 
+ \\ &C_2\sigma_{\max}\sqrt{\tfrac{\log(K/\delta)}{M_{\min}}} \notag 
+ \tfrac{C_3\sigma_{\max}}{k}\big)\Big]
\end{align}
\end{small}where $K$ is the number of classes, $\sigma_{\max} = \max_i \sigma_i$, $\rho_{\max} = \max_i \rho_i$, and $M_{\min} = \min_i M_i$.
\end{corollary}
\noindent The complete proof is provided in the Appendix.

\section{Experiments}
\label{sec:exres}
\subsection{Datasets}
We primarily assess our framework on eleven classification datasets, including automobiles (Cars~\citep{krause20133d}), textures (DTD~\citep{cimpoi2014describing}), human actions (UCF101~\citep{soomro2012ucf101}), aircraft types (Aircraft~\citep{maji2013fine}), satellite imagery (EuroSAT~\citep{helber2019eurosat}), pet breeds (Pets~\citep{parkhi2012cats}), flowers (Flower102~\citep{nilsback2008automated}), food items (Food101~\citep{food101}), scenes (SUN397~\citep{sun397}), and general objects (Caltech101~\citep{fei2004learning} and ImageNet~\citep{deng2009imagenet}). The statistics of all datasets are shown in the Appendix.
\subsection{Baselines} 
We compare our concept-guided Bayesian classification (CGBC) framework against six zero-shot image recognition baselines, involving \textit{CLIP}~\citep{clip}, CLIP + E~\citep{ 
clip}, TPT~\citep{shu2022tpt}, MTA~\citep{matcvpr24}, COLA~\citep{COLA}, C-TPT~\citep{C-tpt}, DiffTPT~\citep{DiffTPT}, ZERO~\citep{zero} and Cupl~\citep{pratt2023does}.  These baselines differ in their choice of prompt templates and test-time augmentation techniques. Additionally, we also compare with CGBC Prior by excluding the likelihood of our framework and averaging the prediction of all concept-enhanced prompts. The details of all baselines are shown in the Appendix.
\subsection{Implementation details}
For concept synthesis, we use GPT-4.1 Turbo to generate 10 concepts per API call. For each class $Y_i$, we generate 50 atomic concepts ($M_A = 50$) and randomly sample 500 composite combinations ($|\hat{\mathcal{C}}_i| = 500$) from the atomic set with 3 atomic concepts per combination. The final compositional concept set $\mathcal{C}_i$ is selected using DPP algorithm with a size of either 16 or 50 ($M_i = 16$ or $50$) in our main experiments. In the Adaptive Soft-Trim Likelihood, we set the outlier threshold as $\lambda = 2.5$ and sigmoid slope as the logit scale of CLIP $k = e^{4.6}$. All baselines are evaluated using their default hyperparameters as reported in their respective papers. Unless otherwise specified, we use CLIP with a ViT-B/16 backbone as the primary VLM. To mitigate the effect of randomness, we report the average top-1 accuracy over three random seeds for all main experiments.
\subsection{Zero-shot Performance Analysis}
\begin{table*}[t!]
\caption{Performance of zero-shot methods on eleven image classification datasets. We highlight the best and second best results by {\bf bolding} and \underline{underlining} them, respectively. In the \textit{Auxiliary} column, the first number indicates the number of views, and the second denotes the number of prompts used by the corresponding method.}
\label{tab:main}
\centering
\begin{adjustbox}{width=\linewidth, center}
\begin{tabular}{cccccccccccccc}
\toprule
Method & SUN397 & Aircraft & EuroSAT & Cars & Food101 & Pets &  Flower102 & Caltech101 & DTD & UCF101 & ImageNet & \textbf{Avg.} & \textbf{Auxiliary}
\\ \midrule 
\textit{CLIP}~\citep{clip} & 62.3 & 23.9 & 42.2 & 65.5 & 82.2 & 87.2 & 67.4 & 92.2 & 44.4 & 64.3 & 66.7 &63.5 & (1, 1) \\
CLIP + E.~\citep{clip} & 65.1 & 23.7 & 47.7 & 66.3 & 82.3 & 87.1 & 66.1 & 93.1 & 44.0 & 65.0 & 68.4  &64.4 & (1, 80)\\
\midrule
TPT~\citep{shu2022tpt} & 65.4 & 23.1 & 42.9 & 66.4 & 84.6 & 87.2 & 68.9 & 94.1 & 47.0 & 68.0 & 68.9 & 65.1& (64, 1)\\
MTA~\citep{matcvpr24}  & 65.0 & 25.3 & 38.7 & \textbf{68.1} & 85.0 & 88.2 & 68.3 & 94.1 & 45.6 & 68.1 & 68.1 & 64.9 &  (64, 1) \\
COLA~\citep{COLA}  & 61.9 & 20.9 & 53.8 & 54.2 & 84.5 & 87.9 & 66.1 & 88.1 & 41.0 & 67.1 & 62.8& 62.6 & (5, 50)\\
C-TPT~\citep{C-tpt}            & 64.8 & 24.0 & 43.2 & 65.8 & 83.7 & 88.2 & 69.8 & 93.6 & 46.0 & 65.7 & 68.5 & 64.9 &(64, 4) \\
DiffTPT~\citep{DiffTPT}         & 65.7 & 25.6 & 43.1 & 67.0 & \textbf{87.2} & 88.2 & 70.1 & 92.5 & 47.0 & 68.2 & \textbf{70.3} & 65.9 & (64, 1)\\
ZERO~\citep{zero}          & 65.0 & 25.2 & 34.3 & 68.0 & \underline{86.5} & 87.8 & 67.7 & 93.7 & 46.1 & 67.8 & 69.3 & 64.7 & (64, 1) \\
\midrule
\multirow{2}{*}{Cupl~\citep{pratt2023does}}& 67.1 & \underline{27.8} & 53.7 & 64.9 & 82.5 & 89.0 & \underline{74.2} & 93.5 & 51.6 & 70.8 & 66.8 & 67.4  & (1, 16) \\
& 67.8 & \textbf{27.9} & 55.3 & 65.4 & 83.1 & 88.9 & \textbf{74.5} & 93.7 & 52.3 & 71.2 & 66.5 & 67.9 & (1, 50) \\
\midrule
CGBC Prior & 68.2 ± 0.1 & 26.0 ± 0.2 & 59.0 ± 0.8 & 66.0 ± 0.2 & 83.5 ± 0.1 & 88.8 ± 0.1 & 71.7 ± 0.2 & 94.2 ± 0.1 & 54.6 ± 0.4 & 71.6 ± 0.3 & 68.2 ± 0.0 & 68.4 ± 0.2 & (1, 16) \\
\rowcolor{lavender1}
\textbf{CGBC} 
 & \underline{68.5} ± 0.0 & 25.9 ± 0.1 & \underline{60.1} ± 0.5 & 66.2 ± 0.1 & 83.5 ± 0.0 & \underline{90.2} ± 0.5 & 73.4 ± 0.1 & \underline{94.3} ± 0.0 & \underline{55.4} ± 0.3 & \underline{71.8} ± 0.3 & 69.2 ± 0.1 & \underline{69.0 }± 0.2 & (1, 16) \\ 
CGBC Prior  & \underline{68.5} ± 0.1 & 26.1 ± 0.2 & 59.3 ± 0.2 & \underline{66.6} ± 0.2 & 83.8 ± 0.0 & 88.5 ± 0.1 & 72.4 ± 0.1 & 94.0 ± 0.1 & 55.1 ± 0.2 & 71.8 ± 0.3 & 68.3 ± 0.1 & 68.6 ± 0.2 & (1, 50) \\
 \rowcolor{lavender1} \textbf{CGBC} 
 & \textbf{68.8} ± 0.7 & 26.0 ± 0.1 & \textbf{60.3} ± 0.4 & \underline{66.6} ± 0.2 & 83.9 ± 0.9 & \textbf{90.7} ± 0.1 & 73.7 ± 0.2 & \textbf{94.4} ± 0.8 & \textbf{55.7} ± 0.4 & \textbf{72.3} ± 0.3 & \underline{69.4} ± 0.0 & \textbf{69.3} ± 0.2 & (1, 50) \\
\bottomrule
\end{tabular}
\end{adjustbox}
\vspace{-15pt}
\end{table*}
\begin{table}[htbp]
\centering
\caption{Averaged top-1 accuracy (\%) on 11 datasets across different VLM architectures. Hyper-parameters are kept constant across all experiments. The best results are bold for each VLM variant.}
\label{tab:diffvlm}
\begin{adjustbox}{max width=\linewidth}
\begin{tabular}{ccccc>{\columncolor{lavender1}}c}
\toprule
VLM Variant & \textit{CLIP} & CLIP + E & MTA~\citep{matcvpr24} & CGBC Prior & \textbf{CGBC} \\
\midrule
RN50     & 55.11 & 56.12 &        57.24       & 60.39 ± 0.19 & \textbf{60.77} ± 0.19 \\
RN101    & 57.38 & 58.43 &         59.34      & 59.78 ± 0.19 & \textbf{61.46} ± 0.19 \\
RN50x4   & 61.21 & 61.78 &        63.28       & 63.60 ± 0.19 & \textbf{64.64} ± 0.19 \\
RN50x16  & 63.69 & 66.00 &         65.38       & 66.46 ± 0.19 & \textbf{66.95} ± 0.19 \\
ViT-B/16 & 63.49 & 64.43 & 64.95         & 68.56 ± 0.18 & \textbf{69.25} ± 0.19 \\
ViT-B/32 & 59.28 & 60.39 &     61.66           & 63.90 ± 0.18 & \textbf{64.39} ± 0.19 \\
ViT-L/14 & 70.58 & 71.36 &          71.71     & 74.39 ± 0.16 & \textbf{75.07} ± 0.16 \\
\bottomrule
\end{tabular}
\end{adjustbox}
\vspace{-15pt}
\end{table}
We compare the zero-shot performance of our CGBC framework with various baselines in Table~\ref{tab:main}, yielding the following key observations: First, prompt-based methods leveraging visual concepts (e.g., CuPL, CGBC Prior) consistently outperform image-augmented baselines (e.g., TPT, MTA) by over 3\% on average across eleven datasets. In contrast, simple prompt ensembling (i.e., CLIP+E) yields marginal improvements, suggesting that LLM-generated class-specific descriptions more effectively activate the internal knowledge of VLMs than class labels alone. Second, our concept synthesis pipeline (CGBC Prior) outperforms CuPL by producing more diverse and recognition-aligned prompts, particularly when using 16 prompts. Empirical analysis shows that CuPL often over-conjoins concepts (typically more than four) using conjunctions (e.g., “and”), which reduces robustness compared to our use of disjunction (“or”). Third, incorporating the adaptive soft-trim likelihood, CGBC further enhances performance by reducing the impact of outlier concepts, underscoring the benefit of refining the prior based on the test image.

We further evaluate the versatility and generalizability of CGBC Prior and CGBC across multiple CLIP variants in Table~\ref{tab:diffvlm}, covering a range of visual encoders and model scales. Both CGBC and its concept prior consistently outperform all three baselines, demonstrating the effectiveness of the concept synthesis pipeline and the adaptive soft-trim likelihood. Specifically, CGBC achieves an average top-1 accuracy gain of ~5\% with ViT-based encoders and ~3\% with RNN-based architectures. Importantly, these improvements persist across increasing VLM capacities, suggesting that larger models more effectively leverage the rich visual concepts synthesized by LLMs. These results highlight CGBC’s potential as a robust, concept-driven framework for a zero-shot recognition paradigm.
\subsection{Concept Proposal Distribution Analysis}
\label{sec:5.5}
\begin{figure}[htbp]
    \centering
        \centering
\includegraphics[width=\linewidth]{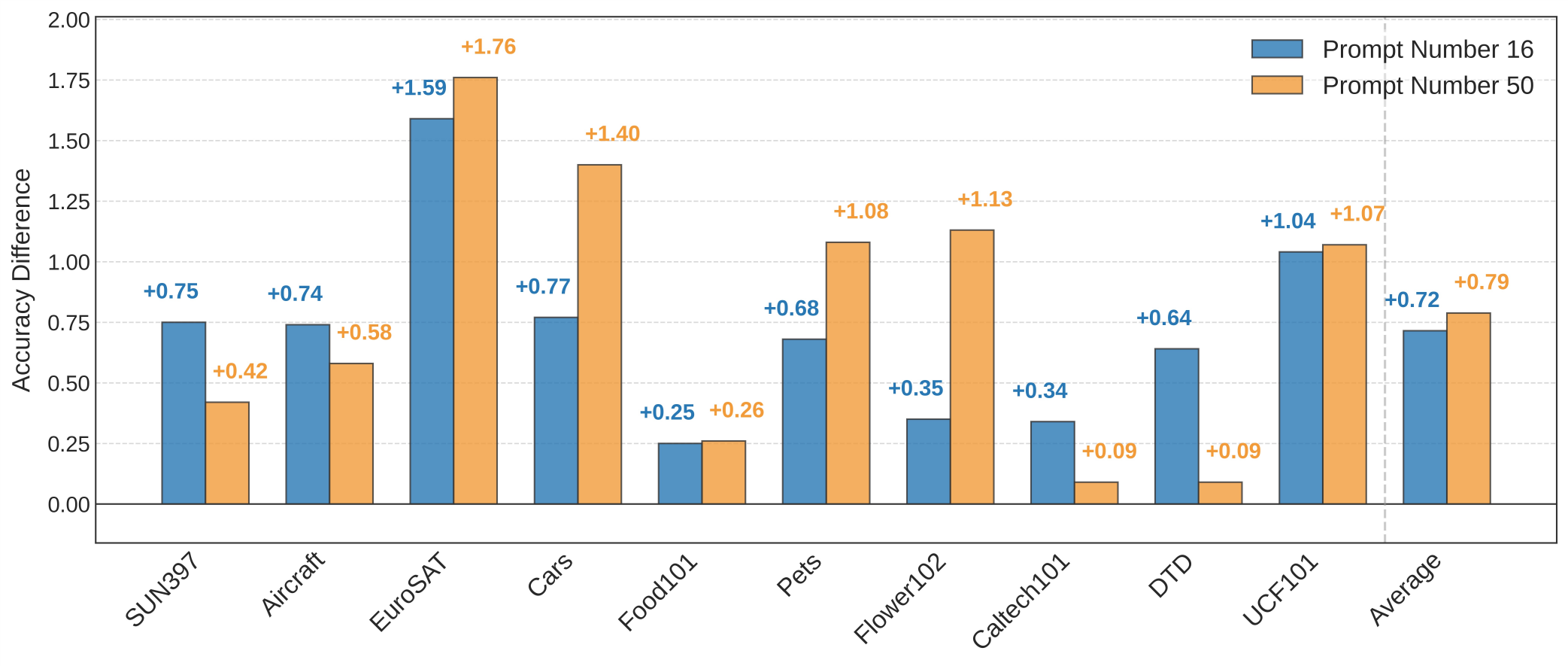}
    \caption{Top-1 Accuracy difference between discriminative (CGBC Prior) and descriptive concepts on ten datasets.} 
    \label{fig:disdes}
    \vspace{-10pt}
\end{figure}
\begin{figure}[htbp]
    \centering
        \centering
\includegraphics[width=\linewidth]{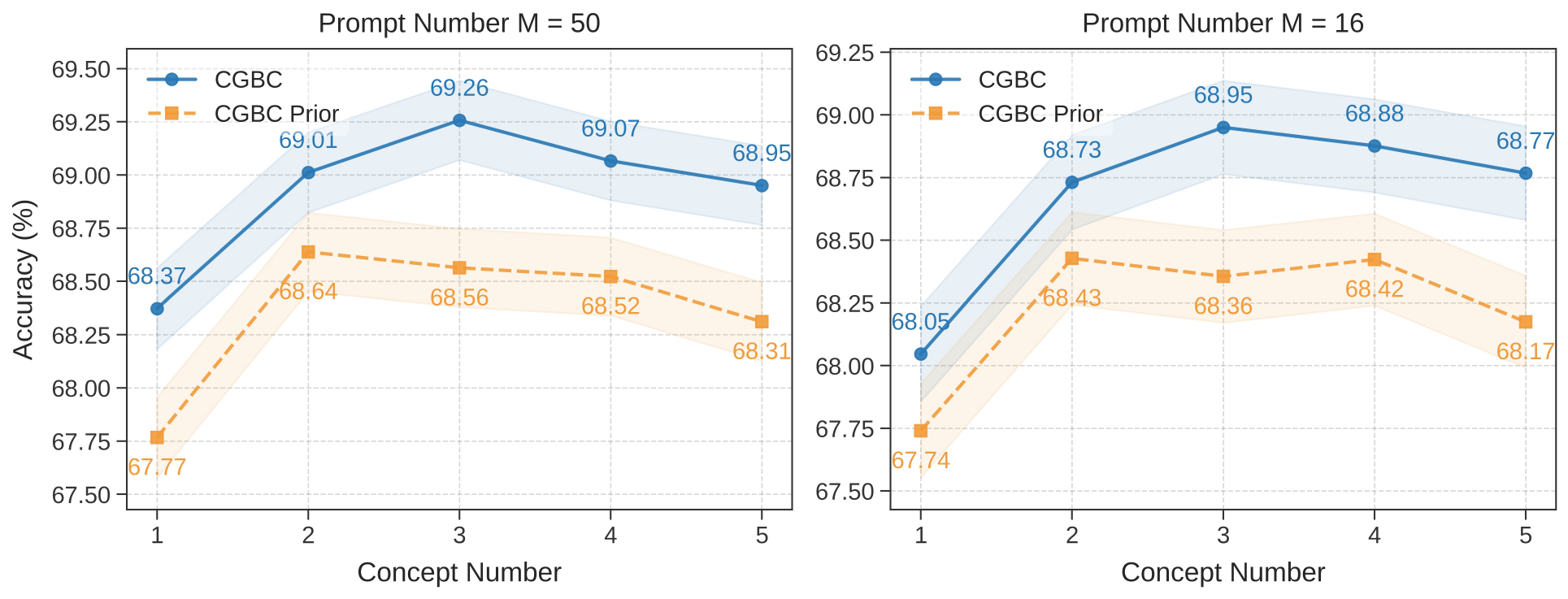}
    \caption{Averaged top-1 accuracy (\%) on eleven datasets using different numbers of concepts per prompt.} 
    \label{fig:com}
\end{figure}
\begin{figure}[htbp]
    \centering
        \centering
\includegraphics[width=\linewidth]{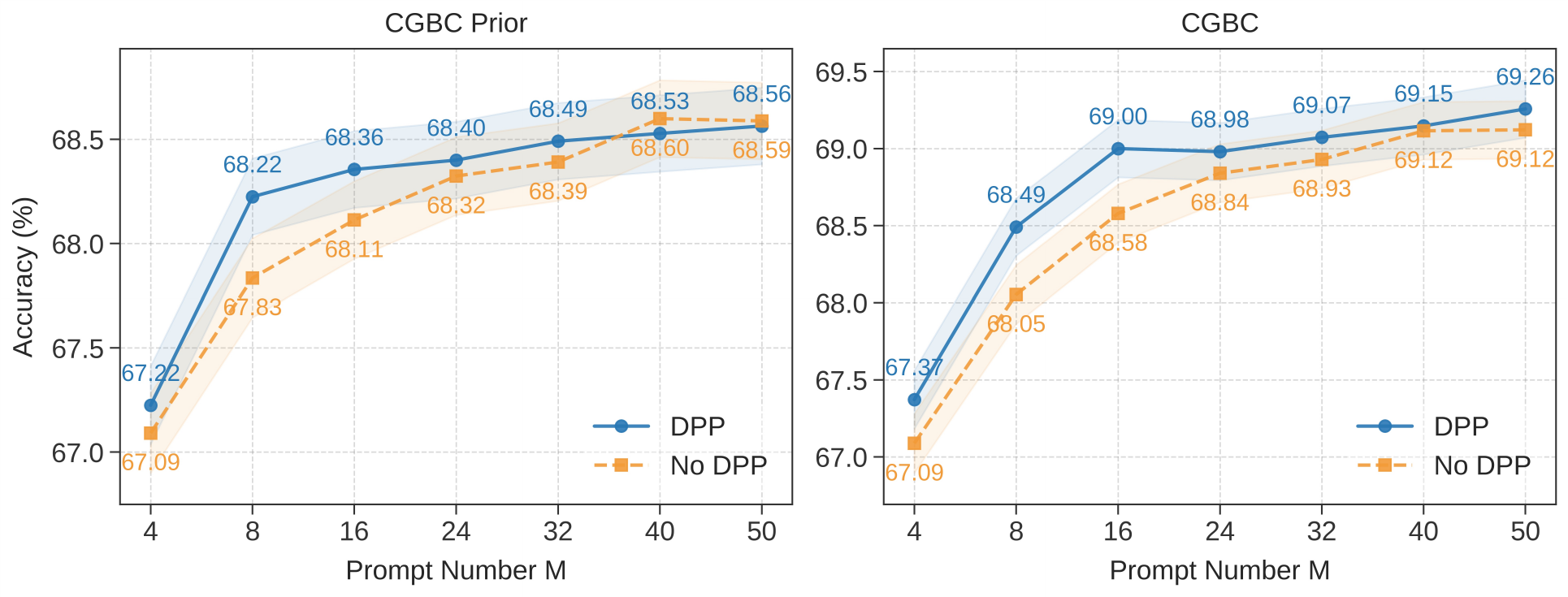}
    \caption{Averaged top-1 accuracy (\%) on eleven datasets with and without DPP under different prompt budgets.} 
    \label{fig:3}
\end{figure}
We empirically validate the importance of modeling the concept proposal distribution using  discriminability, compositionality, and diversity. To assess discriminability, we compare our LLM-driven concept synthesis pipeline with a descriptive-only variant that removes the neighborhood construction stage and replaces contrastive prompts with generic class descriptions (e.g., “Please generate concepts to describe the target class”), keeping all other components unchanged. As shown in Fig.\ref{fig:disdes}, concepts sampled from the discriminative proposal distribution consistently outperform the descriptive baseline, particularly as the number of prompts increases from 16 to 50, highlighting the alignment between discriminative prompting and classification objectives. For compositionality, we vary the number of concepts per prompt during the compositional concept construction stage under two prompt budget settings in Fig. \ref {fig:com}. Performance across eleven datasets initially increases with more concepts per prompt but declines beyond a certain point, peaking at three concepts for CGBC and two for CGBC Prior, indicating that moderate compositionality enhances expressiveness, whereas excessive composition introduces noise. To evaluate diversity, we compare CGBC and CGBC Prior with and without DPP-based sampling in Fig.~\ref{fig:3}. DPP yields greater gains when the number of prompts is small, as diversity helps form a more expressive concept prior under limited budgets. Moreover, all CGBC variants exhibit consistent performance improvements as the number of prompts increases, though gains plateau beyond 24 prompts. This indicates that only a limited number of diverse prompts are sufficient to achieve strong performance in practice, even though Theorem\ref{thm:robust_mean} suggests a larger theoretical minimum. In addition to these quantitative findings, qualitative examples in the appendix further illustrate the necessity of these three aspects in constructing effective concept proposal distributions.
\subsection{Ablation Study}
\begin{table}[htbp]
\centering
\caption{Impact of different LLMs (GPT-4.1 Series) and the size of hard-negative neighborhood set $|L_i|$ in concept synthesis pipeline on the performance of CGBC. \textbf{Avg} is the averaged top-1 accuracy across all five datasets. \textbf{Avg$^\dagger$} is the average excluding EuroSAT.}
\begin{adjustbox}{max width=\linewidth}
\begin{tabular}{ccccccccc}
\toprule
LLMs &  $|L_i|$  & EuroSAT  & DTD  & Aircraft  & Pets  & UCF101 & \textbf{Avg.} & \textbf{Avg.}$^\dagger$ \\
\midrule
\rowcolor{lavender1} Turbo   & 10 & 60.30  & 55.69 & 25.99 & 90.68 & 72.28 & \textbf{60.99} & \textbf{61.16} \\
Mini & 10 & 57.29 & 55.28 & 26.28 & 90.65 & 71.36 & 60.17 & 60.89 \\
Nano & 10 & 58.30 & 54.53 & 24.95 & 90.80 & 70.81 & 59.88 & 60.27 \\
Mini & 5  & 56.30 & 55.56 & 26.11 & 89.70 & 71.55 & 59.84 & 60.73 \\
Mini & 20 & 57.29 & 55.89 & 26.33 & 90.71 & 71.40 & 60.32 & 61.08 \\
Mini & 30 & 57.29 & 55.16 & 26.21 & 90.47 & 71.06 & 60.04 & 60.73 \\
\bottomrule
\end{tabular}
\end{adjustbox}
\vspace{-15pt}
\end{table}
\begin{table}[htbp]
\centering
\caption{Averaged top-1 accuracy (\%) on eleven datasets across different likelihood functions. The best results are bold for each prompt number.}
\label{tab:diffmean}
\begin{adjustbox}{max width=\linewidth}
\begin{tabular}{
c 
c 
>{\columncolor{lavender1}}c 
>{\columncolor{lavender1}}c 
>{\columncolor{lavender1}}c 
>{\columncolor{lavender1}}c 
>{\columncolor{lavender1}}c 
c 
}
\toprule
Prompt & CGBC Prior & CGBC & Median & Hard-Trim & Huber Loss & Cauchy & Confidence \\
\midrule
16 & 68.35 & \textbf{68.95} & 68.77 & 68.40 & 68.53 & 68.39 & 68.45 \\
50 & 68.56 & \textbf{69.25} & 68.88 & 69.07 & 68.94 & 68.87 & 68.61 \\
\bottomrule
\end{tabular}
\end{adjustbox}
\end{table}
We conduct a comprehensive ablation study to evaluate the impact of different LLM backbones, the size of the hard-negative neighborhood set, likelihood functions, and key hyperparameter $(\lambda, k)$ on  our proposed CGBC framework.

As shown in Table~\ref{tab:diffvlm}, we instantiate the concept synthesis pipeline with three versions of GPT-4.1, including Turbo, Mini, and Nano, representing decreasing model capacities, while keeping other settings consistent. Results across five datasets show stronger LLMs yield more effective concept priors, leading to better classification accuracy. This underscores the potential of the concept-based paradigm, benefiting from ongoing advances in LLM capabilities. Notably, even with the smallest model, CGBR consistently outperforms image augmentation-based methods such as MTA in Table \ref{tab:main}, demonstrating the advantage of incorporating external prior beyond the test image and its views. 

As shown in Table \ref{tab:diffvlm}, we examine the impact of the hard-negative neighborhood size $|L_i|$  when using GPT-4.1 Mini for concept synthesis. The results show that performance initially improves with increasing $|L_i|$, peaking at $|L_i| = 20$, and then declines. This trend provides further evidence that discriminative concepts enhance VLM-based image classification. However, overly large neighborhoods introduces noise, ultimately degrading performance.

As shown in Table~\ref{tab:diffmean}, we instantiate the likelihood function in our Bayesian classification framework using various robust mean estimation methods~\citep{huber1992robust}, as well as a confidence-based baseline that selects predictions with maximum similarity. Huber, Cauchy, and our sigmoid method represent soft-trim strategies that down-weight outlier concepts, whereas other methods adopt hard-trim approaches that discard them entirely. All robust estimators consistently outperform naive averaging, confirming the importance of outlier mitigation. In contrast, the confidence-based method performs poorly in the presence of outliers, as it may overemphasize noisy or spurious matches. Notably, our sigmoid-based logistic function achieves the best performance, likely due to its smooth and bounded mapping, which is well-suited to the skewed and long-tailed distribution of similarity scores between prompts and test images.
\begin{table}[t]
\centering
\caption{Averaged top-1 accuracy and contamination rate ($\bar{\rho}$) across eleven datasets under different hyper-parameter combinations of $\lambda$ and $k$. All other settings are fixed.}
\label{tab:hyper}
\begin{adjustbox}{max width=\linewidth}
\begin{tabular}{c|cccccc}
\toprule
\multirow{2}{*}{Metric} 
& \multicolumn{6}{c}{Hyper-parameter group ( $\lambda$, $k$ )} \\
\cmidrule(lr){2-7}
& (2.5, $e^{4.6}$) & (1.5, $e^{4.6}$) & (3.0, $e^{4.6}$) & (2.0, $e^{4.6}$) & (2.5, 0.25) & (2.5, 4) \\
\midrule
\textbf{Avg.}       & 69.25  & 69.07 & 69.19 & 69.14 & 69.11  & 69.16 \\
$\bar{\rho}$        & 0.107      & 0.191    & 0.058        & 0.191      & 0.107      & 0.107  \\
\bottomrule
\end{tabular}
\end{adjustbox}
\vspace{-15pt}
\end{table}

We further investigate the impact of the hyperparameters  $\lambda$ and $k$, controlling the sensitivity of outlier detection and the slope of the sigmoid-based likelihood in Table~\ref{tab:hyper}. Results show that increasing  $\lambda$, performance first improves and then declines, peaking at $\lambda=2.5$, while the average contamination rate consistently decreases across all tasks. Moreover, CGBC achieves the best performance when $k$ is set to the default logit scale of CLIP. These findings indicate that CGBC is robust to hyperparameter variations, requiring minimal tuning for effective deployment.

\section{Conclusion}
Vision-language models-based zero-shot image recognition has garnered increasing attention. While prompt-augmentation methods outperform view-based alternatives, they often rely on heuristic designs and lack systematic evaluation. To address these limitations, we enhance prompts with class-specific concepts and 
rethink zero-shot recognition from a Bayesian perspective. Our formulation underscores the importance of constructing a well-structured concept proposal distribution  and a test image-dependent likelihood function that refines the concept prior. To this end, we propose a multi-stage concept synthesis pipeline that leverages LLMs to generate concepts emphasizing discriminability, compositionality, and diversity. In addition, we introduce a training-free Adaptive Soft-Trim likelihood function to reduce the influence of outlier concepts in a single forward pass. Empirical results across various image recognition benchmarks show that our framework consistently outperforms state-of-the-art baselines.
\section*{Acknowledgment}
We thank the UGC FITE (6460001), RMGS (9229106), the MOE AcRF Tier 1 Seed Grant (RS37/24, \#025041-00001), Singapore, and NSFC (62525107), China for supporting this work.
{
    \small
    \bibliographystyle{ieeenat_fullname}
    \bibliography{main}

\begin{thebibliography}{66}
\providecommand{\natexlab}[1]{#1}
\providecommand{\url}[1]{\texttt{#1}}
\expandafter\ifx\csname urlstyle\endcsname\relax
  \providecommand{\doi}[1]{doi: #1}\else
  \providecommand{\doi}{doi: \begingroup \urlstyle{rm}\Url}\fi

\bibitem[Bossard et~al.(2014)Bossard, Guillaumin, and Van~Gool]{food101}
Lukas Bossard, Matthieu Guillaumin, and Luc Van~Gool.
\newblock Food-101--mining discriminative components with random forests.
\newblock In \emph{Computer vision--ECCV 2014: 13th European conference, zurich, Switzerland, September 6-12, 2014, proceedings, part VI 13}, pages 446--461. Springer, 2014.

\bibitem[Chater and Oaksford(2008)]{chater2008probabilistic}
Nick Chater and Mike Oaksford.
\newblock \emph{The probabilistic mind: Prospects for Bayesian cognitive science}.
\newblock Oxford University Press, USA, 2008.

\bibitem[Chen et~al.(2022)Chen, Yao, Song, Li, Rao, and Zhang]{chen2022plot}
Guangyi Chen, Weiran Yao, Xiangchen Song, Xinyue Li, Yongming Rao, and Kun Zhang.
\newblock Plot: Prompt learning with optimal transport for vision-language models.
\newblock \emph{arXiv preprint arXiv:2210.01253}, 2022.

\bibitem[Cimpoi et~al.(2014)Cimpoi, Maji, Kokkinos, Mohamed, and Vedaldi]{cimpoi2014describing}
Mircea Cimpoi, Subhransu Maji, Iasonas Kokkinos, Sammy Mohamed, and Andrea Vedaldi.
\newblock Describing textures in the wild.
\newblock In \emph{Proceedings of the IEEE conference on computer vision and pattern recognition}, pages 3606--3613, 2014.

\bibitem[Deng et~al.(2009)Deng, Dong, Socher, Li, Li, and Fei-Fei]{deng2009imagenet}
Jia Deng, Wei Dong, Richard Socher, Li-Jia Li, Kai Li, and Li Fei-Fei.
\newblock Imagenet: A large-scale hierarchical image database.
\newblock In \emph{2009 IEEE conference on computer vision and pattern recognition}, pages 248--255. Ieee, 2009.

\bibitem[Diakonikolas et~al.(2018)Diakonikolas, Kamath, Kane, Li, Moitra, and Stewart]{diakonikolas2018robustly}
Ilias Diakonikolas, Gautam Kamath, Daniel~M Kane, Jerry Li, Ankur Moitra, and Alistair Stewart.
\newblock Robustly learning a gaussian: Getting optimal error, efficiently.
\newblock In \emph{Proceedings of the Twenty-Ninth Annual ACM-SIAM Symposium on Discrete Algorithms}, pages 2683--2702. SIAM, 2018.

\bibitem[Diakonikolas et~al.(2023)Diakonikolas, Kane, Pensia, and Pittas]{diakonikolas2023near}
Ilias Diakonikolas, Daniel Kane, Ankit Pensia, and Thanasis Pittas.
\newblock Near-optimal algorithms for gaussians with huber contamination: Mean estimation and linear regression.
\newblock \emph{Advances in Neural Information Processing Systems}, 36:\penalty0 43384--43422, 2023.

\bibitem[Ellis(2024)]{ellis2024human}
Kevin Ellis.
\newblock Human-like few-shot learning via bayesian reasoning over natural language.
\newblock \emph{Advances in Neural Information Processing Systems}, 36, 2024.

\bibitem[Ellis et~al.(2021)Ellis, Wong, Nye, Sabl{\'e}-Meyer, Morales, Hewitt, Cary, Solar-Lezama, and Tenenbaum]{ellis2021dreamcoder}
Kevin Ellis, Catherine Wong, Maxwell Nye, Mathias Sabl{\'e}-Meyer, Lucas Morales, Luke Hewitt, Luc Cary, Armando Solar-Lezama, and Joshua~B Tenenbaum.
\newblock Dreamcoder: Bootstrapping inductive program synthesis with wake-sleep library learning.
\newblock In \emph{Proceedings of the 42nd acm sigplan international conference on programming language design and implementation}, pages 835--850, 2021.

\bibitem[Fang et~al.(2023)Fang, Wang, Xie, Sun, Wu, Wang, Huang, Wang, and Cao]{evaclip}
Yuxin Fang, Wen Wang, Binhui Xie, Quan Sun, Ledell Wu, Xinggang Wang, Tiejun Huang, Xinlong Wang, and Yue Cao.
\newblock Eva: Exploring the limits of masked visual representation learning at scale.
\newblock In \emph{Proceedings of the IEEE/CVF Conference on Computer Vision and Pattern Recognition}, pages 19358--19369, 2023.

\bibitem[Farina et~al.(2024)Farina, Franchi, Iacca, Mancini, and Ricci]{zero}
Matteo Farina, Gianni Franchi, Giovanni Iacca, Massimiliano Mancini, and Elisa Ricci.
\newblock Frustratingly easy test-time adaptation of vision-language models.
\newblock \emph{Advances in Neural Information Processing Systems}, 37:\penalty0 129062--129093, 2024.

\bibitem[Fei-Fei et~al.(2004)Fei-Fei, Fergus, and Perona]{fei2004learning}
Li Fei-Fei, Rob Fergus, and Pietro Perona.
\newblock Learning generative visual models from few training examples: An incremental bayesian approach tested on 101 object categories.
\newblock In \emph{2004 conference on computer vision and pattern recognition workshop}, pages 178--178. IEEE, 2004.

\bibitem[Feng et~al.(2023{\natexlab{a}})Feng, Yu, Liu, Khan, and Zuo]{DiffTPT}
Chun-Mei Feng, Kai Yu, Yong Liu, Salman Khan, and Wangmeng Zuo.
\newblock Diverse data augmentation with diffusions for effective test-time prompt tuning.
\newblock In \emph{Proceedings of the IEEE/CVF International Conference on Computer Vision}, pages 2704--2714, 2023{\natexlab{a}}.

\bibitem[Feng et~al.(2023{\natexlab{b}})Feng, Yu, Liu, Khan, and Zuo]{feng2023diverse}
Chun-Mei Feng, Kai Yu, Yong Liu, Salman Khan, and Wangmeng Zuo.
\newblock Diverse data augmentation with diffusions for effective test-time prompt tuning.
\newblock In \emph{Proceedings of the IEEE/CVF International Conference on Computer Vision}, pages 2704--2714, 2023{\natexlab{b}}.

\bibitem[Gao et~al.(2024)Gao, Geng, Zhang, Ma, Fang, Zhang, Li, and Qiao]{clipadapter}
Peng Gao, Shijie Geng, Renrui Zhang, Teli Ma, Rongyao Fang, Yongfeng Zhang, Hongsheng Li, and Yu Qiao.
\newblock Clip-adapter: Better vision-language models with feature adapters.
\newblock \emph{International Journal of Computer Vision}, 132\penalty0 (2):\penalty0 581--595, 2024.

\bibitem[Ge et~al.(2023)Ge, Ren, Gallagher, Wang, Yang, Adam, Itti, Lakshminarayanan, and Zhao]{zeroshotcvpr23}
Yunhao Ge, Jie Ren, Andrew Gallagher, Yuxiao Wang, Ming-Hsuan Yang, Hartwig Adam, Laurent Itti, Balaji Lakshminarayanan, and Jiaping Zhao.
\newblock Improving zero-shot generalization and robustness of multi-modal models.
\newblock In \emph{Proceedings of the IEEE/CVF conference on computer vision and pattern recognition}, pages 11093--11101, 2023.

\bibitem[Glanois et~al.(2022)Glanois, Jiang, Feng, Weng, Zimmer, Li, Liu, and Hao]{glanois2022neuro}
Claire Glanois, Zhaohui Jiang, Xuening Feng, Paul Weng, Matthieu Zimmer, Dong Li, Wulong Liu, and Jianye Hao.
\newblock Neuro-symbolic hierarchical rule induction.
\newblock In \emph{International Conference on Machine Learning}, pages 7583--7615. PMLR, 2022.

\bibitem[Helber et~al.(2019)Helber, Bischke, Dengel, and Borth]{helber2019eurosat}
Patrick Helber, Benjamin Bischke, Andreas Dengel, and Damian Borth.
\newblock Eurosat: A novel dataset and deep learning benchmark for land use and land cover classification.
\newblock \emph{IEEE Journal of Selected Topics in Applied Earth Observations and Remote Sensing}, 12\penalty0 (7):\penalty0 2217--2226, 2019.

\bibitem[Hendrycks et~al.(2021{\natexlab{a}})Hendrycks, Basart, Mu, Kadavath, Wang, Dorundo, Desai, Zhu, Parajuli, Guo, et~al.]{imagenetr}
Dan Hendrycks, Steven Basart, Norman Mu, Saurav Kadavath, Frank Wang, Evan Dorundo, Rahul Desai, Tyler Zhu, Samyak Parajuli, Mike Guo, et~al.
\newblock The many faces of robustness: A critical analysis of out-of-distribution generalization.
\newblock In \emph{Proceedings of the IEEE/CVF international conference on computer vision}, pages 8340--8349, 2021{\natexlab{a}}.

\bibitem[Hendrycks et~al.(2021{\natexlab{b}})Hendrycks, Zhao, Basart, Steinhardt, and Song]{imageneta}
Dan Hendrycks, Kevin Zhao, Steven Basart, Jacob Steinhardt, and Dawn Song.
\newblock Natural adversarial examples.
\newblock In \emph{Proceedings of the IEEE/CVF conference on computer vision and pattern recognition}, pages 15262--15271, 2021{\natexlab{b}}.

\bibitem[Huang et~al.(2025)Huang, Jiang, Jiang, Li, Khan, and Wang]{huang2025cosmic}
Fanding Huang, Jingyan Jiang, Qinting Jiang, Hebei Li, Faisal~Nadeem Khan, and Zhi Wang.
\newblock Cosmic: Clique-oriented semantic multi-space integration for robust clip test-time adaptation.
\newblock In \emph{Proceedings of the Computer Vision and Pattern Recognition Conference}, pages 9772--9781, 2025.

\bibitem[Huang et~al.(2022)Huang, Chu, and Wei]{huang2022unsupervised}
Tony Huang, Jack Chu, and Fangyun Wei.
\newblock Unsupervised prompt learning for vision-language models.
\newblock \emph{arXiv preprint arXiv:2204.03649}, 2022.

\bibitem[Huber(1992)]{huber1992robust}
Peter~J Huber.
\newblock Robust estimation of a location parameter.
\newblock In \emph{Breakthroughs in statistics: Methodology and distribution}, pages 492--518. Springer, 1992.

\bibitem[Kahneman(2011)]{kahneman2011thinking}
Daniel Kahneman.
\newblock Thinking, fast and slow.
\newblock \emph{Farrar, Straus and Giroux}, 2011.

\bibitem[Kloek and Van~Dijk(1978)]{kloek1978bayesian}
Teun Kloek and Herman~K Van~Dijk.
\newblock Bayesian estimates of equation system parameters: an application of integration by monte carlo.
\newblock \emph{Econometrica: Journal of the Econometric Society}, pages 1--19, 1978.

\bibitem[Krause et~al.(2013)Krause, Stark, Deng, and Fei-Fei]{krause20133d}
Jonathan Krause, Michael Stark, Jia Deng, and Li Fei-Fei.
\newblock 3d object representations for fine-grained categorization.
\newblock In \emph{Proceedings of the IEEE international conference on computer vision workshops}, pages 554--561, 2013.

\bibitem[Kulesza et~al.(2012)Kulesza, Taskar, et~al.]{kulesza2012determinantal}
Alex Kulesza, Ben Taskar, et~al.
\newblock Determinantal point processes for machine learning.
\newblock \emph{Foundations and Trends{\textregistered} in Machine Learning}, 5\penalty0 (2--3):\penalty0 123--286, 2012.

\bibitem[Lafon et~al.(2025)Lafon, Hakim, Rambour, Desrosier, and Thome]{lafon2025cliptta}
Marc Lafon, Gustavo Adolfo~Vargas Hakim, Cl{\'e}ment Rambour, Christian Desrosier, and Nicolas Thome.
\newblock Cliptta: Robust contrastive vision-language test-time adaptation.
\newblock \emph{arXiv preprint arXiv:2507.14312}, 2025.

\bibitem[Lester et~al.(2021)Lester, Al-Rfou, and Constant]{lester2021power}
Brian Lester, Rami Al-Rfou, and Noah Constant.
\newblock The power of scale for parameter-efficient prompt tuning.
\newblock \emph{arXiv preprint arXiv:2104.08691}, 2021.

\bibitem[Li et~al.(2023)Li, Li, Savarese, and Hoi]{blip2}
Junnan Li, Dongxu Li, Silvio Savarese, and Steven Hoi.
\newblock Blip-2: Bootstrapping language-image pre-training with frozen image encoders and large language models.
\newblock In \emph{International conference on machine learning}, pages 19730--19742. PMLR, 2023.

\bibitem[Liu et~al.(2023{\natexlab{a}})Liu, Wang, and Li]{liu2023interpretable}
Hui Liu, Wenya Wang, and Haoliang Li.
\newblock Interpretable multimodal misinformation detection with logic reasoning.
\newblock \emph{arXiv preprint arXiv:2305.05964}, 2023{\natexlab{a}}.

\bibitem[Liu et~al.(2024)Liu, Li, Wu, and Lee]{llava}
Haotian Liu, Chunyuan Li, Qingyang Wu, and Yong~Jae Lee.
\newblock Visual instruction tuning.
\newblock \emph{Advances in neural information processing systems}, 36, 2024.

\bibitem[Liu et~al.(2023{\natexlab{b}})Liu, Du, Li, Tenenbaum, and Torralba]{liu2023unsupervised}
Nan Liu, Yilun Du, Shuang Li, Joshua~B Tenenbaum, and Antonio Torralba.
\newblock Unsupervised compositional concepts discovery with text-to-image generative models.
\newblock In \emph{Proceedings of the IEEE/CVF International Conference on Computer Vision}, pages 2085--2095, 2023{\natexlab{b}}.

\bibitem[Ma et~al.(2024)Ma, Zhang, Guo, and Xu]{ma2024swapprompt}
Xiaosong Ma, Jie Zhang, Song Guo, and Wenchao Xu.
\newblock Swapprompt: Test-time prompt adaptation for vision-language models.
\newblock \emph{Advances in Neural Information Processing Systems}, 36, 2024.

\bibitem[Maji et~al.(2013)Maji, Rahtu, Kannala, Blaschko, and Vedaldi]{maji2013fine}
Subhransu Maji, Esa Rahtu, Juho Kannala, Matthew Blaschko, and Andrea Vedaldi.
\newblock Fine-grained visual classification of aircraft.
\newblock \emph{arXiv preprint arXiv:1306.5151}, 2013.

\bibitem[Martin et~al.(2024)Martin, Huang, Shakeri, Pesquet, and Ben~Ayed]{martin2024transductive}
S{\'e}gol{\`e}ne Martin, Yunshi Huang, Fereshteh Shakeri, Jean-Christophe Pesquet, and Ismail Ben~Ayed.
\newblock Transductive zero-shot and few-shot clip.
\newblock In \emph{Proceedings of the IEEE/CVF Conference on Computer Vision and Pattern Recognition}, pages 28816--28826, 2024.

\bibitem[Nilsback and Zisserman(2008)]{nilsback2008automated}
Maria-Elena Nilsback and Andrew Zisserman.
\newblock Automated flower classification over a large number of classes.
\newblock In \emph{2008 Sixth Indian conference on computer vision, graphics \& image processing}, pages 722--729. IEEE, 2008.

\bibitem[Novack et~al.(2023)Novack, McAuley, Lipton, and Garg]{novack2023chilszeroshot}
Zachary Novack, Julian McAuley, Zachary~Chase Lipton, and Saurabh Garg.
\newblock Chils: Zero-shot image classification with hierarchical label sets.
\newblock In \emph{International Conference on Machine Learning}, pages 26342--26362. PMLR, 2023.

\bibitem[Parkhi et~al.(2012)Parkhi, Vedaldi, Zisserman, and Jawahar]{parkhi2012cats}
Omkar~M Parkhi, Andrea Vedaldi, Andrew Zisserman, and CV Jawahar.
\newblock Cats and dogs.
\newblock In \emph{2012 IEEE conference on computer vision and pattern recognition}, pages 3498--3505. IEEE, 2012.

\bibitem[Pratt et~al.(2023)Pratt, Covert, Liu, and Farhadi]{pratt2023does}
Sarah Pratt, Ian Covert, Rosanne Liu, and Ali Farhadi.
\newblock What does a platypus look like? generating customized prompts for zero-shot image classification.
\newblock In \emph{Proceedings of the IEEE/CVF International Conference on Computer Vision}, pages 15691--15701, 2023.

\bibitem[Qu et~al.(2020)Qu, Chen, Xhonneux, Bengio, and Tang]{qu2020rnnlogic}
Meng Qu, Junkun Chen, Louis-Pascal Xhonneux, Yoshua Bengio, and Jian Tang.
\newblock Rnnlogic: Learning logic rules for reasoning on knowledge graphs.
\newblock \emph{arXiv preprint arXiv:2010.04029}, 2020.

\bibitem[Radford et~al.(2021)Radford, Kim, Hallacy, Ramesh, Goh, Agarwal, Sastry, Askell, Mishkin, Clark, Krueger, and Sutskever]{clip}
Alec Radford, Jong~Wook Kim, Chris Hallacy, Aditya Ramesh, Gabriel Goh, Sandhini Agarwal, Girish Sastry, Amanda Askell, Pamela Mishkin, Jack Clark, Gretchen Krueger, and Ilya Sutskever.
\newblock Learning transferable visual models from natural language supervision.
\newblock 139:\penalty0 8748--8763, 2021.

\bibitem[Recht et~al.(2019)Recht, Roelofs, Schmidt, and Shankar]{imagenetv2}
Benjamin Recht, Rebecca Roelofs, Ludwig Schmidt, and Vaishaal Shankar.
\newblock Do imagenet classifiers generalize to imagenet?
\newblock In \emph{International conference on machine learning}, pages 5389--5400. PMLR, 2019.

\bibitem[Saad et~al.(2019)Saad, Cusumano-Towner, Schaechtle, Rinard, and Mansinghka]{saad2019bayesian}
Feras~A Saad, Marco~F Cusumano-Towner, Ulrich Schaechtle, Martin~C Rinard, and Vikash~K Mansinghka.
\newblock Bayesian synthesis of probabilistic programs for automatic data modeling.
\newblock \emph{Proceedings of the ACM on Programming Languages}, 3\penalty0 (POPL):\penalty0 1--32, 2019.

\bibitem[Shang et~al.(2024)Shang, Zhou, Zhang, Ni, Yang, and Wang]{shang2024incremental}
Chenming Shang, Shiji Zhou, Hengyuan Zhang, Xinzhe Ni, Yujiu Yang, and Yuwang Wang.
\newblock Incremental residual concept bottleneck models.
\newblock In \emph{Proceedings of the IEEE/CVF Conference on Computer Vision and Pattern Recognition}, pages 11030--11040, 2024.

\bibitem[Shu et~al.(2022)Shu, Nie, Huang, Yu, Goldstein, Anandkumar, and Xiao]{shu2022tpt}
Manli Shu, Weili Nie, De-An Huang, Zhiding Yu, Tom Goldstein, Anima Anandkumar, and Chaowei Xiao.
\newblock Test-time prompt tuning for zero-shot generalization in vision-language models.
\newblock \emph{Advances in Neural Information Processing Systems}, 35:\penalty0 14274--14289, 2022.

\bibitem[Silva-Rodriguez et~al.(2024)Silva-Rodriguez, Hajimiri, Ben~Ayed, and Dolz]{silva2024closer}
Julio Silva-Rodriguez, Sina Hajimiri, Ismail Ben~Ayed, and Jose Dolz.
\newblock A closer look at the few-shot adaptation of large vision-language models.
\newblock In \emph{Proceedings of the IEEE/CVF Conference on Computer Vision and Pattern Recognition}, pages 23681--23690, 2024.

\bibitem[Soomro(2012)]{soomro2012ucf101}
K Soomro.
\newblock Ucf101: A dataset of 101 human actions classes from videos in the wild.
\newblock \emph{arXiv preprint arXiv:1212.0402}, 2012.

\bibitem[Tian et~al.(2020)Tian, Ellis, Kryven, and Tenenbaum]{tian2020learning}
Lucas Tian, Kevin Ellis, Marta Kryven, and Josh Tenenbaum.
\newblock Learning abstract structure for drawing by efficient motor program induction.
\newblock \emph{Advances in Neural Information Processing Systems}, 33:\penalty0 2686--2697, 2020.

\bibitem[Udandarao et~al.(2023)Udandarao, Gupta, and Albanie]{udandarao2023sus}
Vishaal Udandarao, Ankush Gupta, and Samuel Albanie.
\newblock Sus-x: Training-free name-only transfer of vision-language models.
\newblock In \emph{Proceedings of the IEEE/CVF International Conference on Computer Vision}, pages 2725--2736, 2023.

\bibitem[Wang et~al.(2019)Wang, Ge, Lipton, and Xing]{imagenetsketch}
Haohan Wang, Songwei Ge, Zachary Lipton, and Eric~P Xing.
\newblock Learning robust global representations by penalizing local predictive power.
\newblock \emph{Advances in Neural Information Processing Systems}, 32, 2019.

\bibitem[Wang et~al.(2023)Wang, Liang, He, Xu, Wang, and Tan]{wang2023improving}
Zhengbo Wang, Jian Liang, Ran He, Nan Xu, Zilei Wang, and Tieniu Tan.
\newblock Improving zero-shot generalization for clip with synthesized prompts.
\newblock In \emph{Proceedings of the IEEE/CVF International Conference on Computer Vision}, pages 3032--3042, 2023.

\bibitem[Xiao et~al.(2010)Xiao, Hays, Ehinger, Oliva, and Torralba]{sun397}
Jianxiong Xiao, James Hays, Krista~A Ehinger, Aude Oliva, and Antonio Torralba.
\newblock Sun database: Large-scale scene recognition from abbey to zoo.
\newblock In \emph{2010 IEEE computer society conference on computer vision and pattern recognition}, pages 3485--3492. IEEE, 2010.

\bibitem[Xiao et~al.(2024)Xiao, Mao, Zhang, He, and Cambria]{xiao2024vanessa}
Luwei Xiao, Rui Mao, Xulang Zhang, Liang He, and Erik Cambria.
\newblock Vanessa: Visual connotation and aesthetic attributes understanding network for multimodal aspect-based sentiment analysis.
\newblock In \emph{Findings of the Association for Computational Linguistics: EMNLP 2024}, pages 11486--11500, 2024.

\bibitem[Xiao et~al.(2025)Xiao, Mao, Zhao, Lin, Jia, He, and Cambria]{xiao2025exploring}
Luwei Xiao, Rui Mao, Shuai Zhao, Qika Lin, Yanhao Jia, Liang He, and Erik Cambria.
\newblock Exploring cognitive and aesthetic causality for multimodal aspect-based sentiment analysis.
\newblock \emph{IEEE Transactions on Affective Computing}, 2025.

\bibitem[Yin et~al.(2026)Yin, Xiao, Feng, Chen, Zhu, Luo, Alamri, Mao, and Cambria]{yin2026aosnet}
Zhengnan Yin, Luwei Xiao, Xuan Feng, Yiwei Chen, Xianxun Zhu, Cai Luo, Faten~S Alamri, Rui Mao, and Erik Cambria.
\newblock Aosnet-sec: Aperture--orientation--spectrum fusion with statistical markov repair for trustworthy super-resolution.
\newblock \emph{Pattern Recognition}, page 113346, 2026.

\bibitem[Yoon et~al.(2024)Yoon, Yoon, Tee, Hasegawa-Johnson, Li, and Yoo]{C-tpt}
Hee~Suk Yoon, Eunseop Yoon, Joshua Tian~Jin Tee, Mark Hasegawa-Johnson, Yingzhen Li, and Chang~D Yoo.
\newblock C-tpt: Calibrated test-time prompt tuning for vision-language models via text feature dispersion.
\newblock \emph{arXiv preprint arXiv:2403.14119}, 2024.

\bibitem[Yu et~al.(2023)Yu, Lu, Jin, Chen, and Wang]{yu2023task}
Tao Yu, Zhihe Lu, Xin Jin, Zhibo Chen, and Xinchao Wang.
\newblock Task residual for tuning vision-language models.
\newblock In \emph{Proceedings of the IEEE/CVF Conference on Computer Vision and Pattern Recognition}, pages 10899--10909, 2023.

\bibitem[Zanella and Ben~Ayed(2024)]{matcvpr24}
Maxime Zanella and Ismail Ben~Ayed.
\newblock On the test-time zero-shot generalization of vision-language models: Do we really need prompt learning?
\newblock In \emph{Proceedings of the IEEE/CVF Conference on Computer Vision and Pattern Recognition}, pages 23783--23793, 2024.

\bibitem[Zanella et~al.(2024)Zanella, G{\'e}rin, and Ayed]{zanellaboostingnips24}
Maxime Zanella, Beno{\^\i}t G{\'e}rin, and Ismail~Ben Ayed.
\newblock Boosting vision-language models with transduction.
\newblock \emph{arXiv preprint arXiv:2406.01837}, 2024.

\bibitem[Zhang et~al.(2022)Zhang, Zhang, Fang, Gao, Li, Dai, Qiao, and Li]{tipadapter}
Renrui Zhang, Wei Zhang, Rongyao Fang, Peng Gao, Kunchang Li, Jifeng Dai, Yu Qiao, and Hongsheng Li.
\newblock Tip-adapter: Training-free adaption of clip for few-shot classification.
\newblock In \emph{European conference on computer vision}, pages 493--510. Springer, 2022.

\bibitem[Zhang et~al.(2023)Zhang, Hu, Li, Huang, Deng, Qiao, Gao, and Li]{zhang2023prompt}
Renrui Zhang, Xiangfei Hu, Bohao Li, Siyuan Huang, Hanqiu Deng, Yu Qiao, Peng Gao, and Hongsheng Li.
\newblock Prompt, generate, then cache: Cascade of foundation models makes strong few-shot learners.
\newblock In \emph{Proceedings of the IEEE/CVF Conference on Computer Vision and Pattern Recognition}, pages 15211--15222, 2023.

\bibitem[Zhou et~al.(2022{\natexlab{a}})Zhou, Yang, Loy, and Liu]{cocoop}
Kaiyang Zhou, Jingkang Yang, Chen~Change Loy, and Ziwei Liu.
\newblock Conditional prompt learning for vision-language models.
\newblock In \emph{Proceedings of the IEEE/CVF conference on computer vision and pattern recognition}, pages 16816--16825, 2022{\natexlab{a}}.

\bibitem[Zhou et~al.(2022{\natexlab{b}})Zhou, Yang, Loy, and Liu]{coop}
Kaiyang Zhou, Jingkang Yang, Chen~Change Loy, and Ziwei Liu.
\newblock Learning to prompt for vision-language models.
\newblock \emph{International Journal of Computer Vision}, 130\penalty0 (9):\penalty0 2337--2348, 2022{\natexlab{b}}.

\bibitem[Zhu et~al.(2023)Zhu, Niu, Han, Wu, and Zhang]{zhu2023prompt}
Beier Zhu, Yulei Niu, Yucheng Han, Yue Wu, and Hanwang Zhang.
\newblock Prompt-aligned gradient for prompt tuning.
\newblock In \emph{Proceedings of the IEEE/CVF International Conference on Computer Vision}, pages 15659--15669, 2023.

\bibitem[Zhu et~al.(2025)Zhu, Zhu, Wang, Zhao, and Zhang]{COLA}
Xingyu Zhu, Beier Zhu, Shuo Wang, Kesen Zhao, and Hanwang Zhang.
\newblock Enhancing clip robustness via cross-modality alignment.
\newblock \emph{arXiv preprint arXiv:2510.24038}, 2025.

\end{thebibliography}
}
\setcounter{page}{1}
\maketitlesupplementary
\section{Experimental Settings}
\begin{table}[h]
\centering
\caption{\centering Statistics of fifteen datasets used in our work.}
\label{tab:datasets_details}
\begin{adjustbox}{width=0.7\linewidth, center}
    \begin{tabular}{l@{\hskip 1em}c@{\hskip 1em}c}
    \toprule
        \textbf{Dataset} & \textbf{Number of Classes} & \textbf{Test Set Size} \\
    \midrule
    SUN397~\citep{sun397} & 397 & 19,850 \\
    Aircraft~\citep{maji2013fine} & 100 & 3,333 \\
    EuroSAT~\citep{helber2019eurosat} & 10 & 8,100 \\
    StanfordCars~\citep{krause20133d} & 196 & 8,041 \\
    Food101~\citep{food101} & 101 & 30,300 \\
    OxfordPets~\citep{parkhi2012cats} & 37 & 3,669 \\
    Flower102~\citep{nilsback2008automated} & 102 & 2,463 \\
    Caltech101~\citep{fei2004learning} & 100 & 2,465 \\
    DTD~\citep{cimpoi2014describing} & 47 & 1,692 \\
    UCF101~\citep{soomro2012ucf101} & 101 & 3,783 \\
    ImageNet~\citep{deng2009imagenet} & 1,000 & 50,000 \\
    ImageNet-A~\citep{imageneta} & 200 & 7,500 \\
    ImageNet-V2~\citep{imagenetv2} & 1,000 & 10,000 \\
    ImageNet-R~\citep{imagenetr} & 200 & 30,000 \\
    ImageNet-Sketch~\citep{imagenetsketch} & 1,000 & 50,889 \\
    \bottomrule
    \end{tabular}
\end{adjustbox}
\end{table}
\subsection{Datasets}
We primarily assess our framework on eleven classification datasets, including automobiles (Cars~\citep{krause20133d}), textures (DTD~\citep{cimpoi2014describing}), human actions (UCF101~\citep{soomro2012ucf101}), aircraft types (Aircraft~\citep{maji2013fine}), satellite imagery (EuroSAT~\citep{helber2019eurosat}), pet breeds (Pets~\citep{parkhi2012cats}), flowers (Flower102~\citep{nilsback2008automated}), food items (Food101~\citep{food101}), scenes (SUN397~\citep{sun397}), and general objects (Caltech101~\citep{fei2004learning} and ImageNet~\citep{deng2009imagenet}). We additionally evaluate the robustness of our framework to natural distribution shifts on four ImageNet variants, involving ImageNet-A~\citep{imageneta}, ImageNet-V2~\citep{imagenetv2}, ImageNet-R~\citep{imagenetr}, ImageNet-Sketch~\citep{imagenetsketch} as out-of-distribution (OOD) data for ImageNet~\citep{deng2009imagenet}. The dataset statistics are shown in Table \ref{tab:datasets_details}.
\subsection{Details of baselines}
\textit{CLIP} and CLIP + E are two versions adopted by~\citet{clip} while the former uses a default prompt ``A photo of a \{class\}." and the latter uses the ensemble of 80 hand-crafted prompts. TPT and MTA generate multiple augmented views of test images to enhance prediction robustness. TPT minimizes the entropy of the prediction distributions among all views via prompt tuning. MTA, on the other hand, optimizes an importance score for each view to find the mode of the density of all views and enforces similar importance scores for views with similar prediction distributions. COLA enhances robustness by projecting adversarial image features into a text-induced subspace and refining alignment through optimal transport. CuPL uses five fixed questions (e.g., ``What does a \{class\} look like?”) to prompt an LLM and treats the responses as textual prompts. C-TPT~\citep{C-tpt} jointly optimizes prompts during test time to improve calibration by maximizing the Average Text Feature Dispersion, motivated by the observation that prompts inducing higher text-feature dispersion yield lower calibration error. DiffTPT~\citep{DiffTPT}  leverages a pre-trained diffusion model to synthesize diverse augmented views for each test image and applies a cosine-similarity filtration to remove spurious generations, thereby balancing augmentation diversity and prediction fidelity for test-time prompt tuning. ZERO~\citep{zero} is a training-free test-time adaptation baseline that augments the test image, keeps the most confident views, then sets the Softmax temperature to zero so that marginalization effectively becomes voting over confident predictions, requiring only a single batched forward pass through the vision encoder and no backpropagation.
\subsection{Implementation details}
For concept synthesis, we use GPT-4.1 Turbo to generate 10 concepts per API call. For each class $Y_i$, we generate 50 atomic concepts ($M_A = 50$) and randomly sample 500 composite combinations ($|\hat{\mathcal{C}}_i| = 500$) from the atomic set with 3 atomic concepts per combination. The final compositional concept set $\mathcal{C}_i$ is selected using DPP algorithm with a size of either 16 or 50 ($M_i = 16$ or $50$) in our main experiments. The prompt used by the LLM for the discriminative concept generation utilized in CGBC is provided in Table~\ref{prompt:congen}, while the prompt for the descriptive baseline, as discussed in Sec.~\ref{sec:5.5}, is detailed in Table~\ref{prompt:desgen}.

In the Adaptive Soft-Trim Likelihood, we set the outlier threshold as $\lambda = 2.5$ and sigmoid slope as the logit scale of CLIP $k = e^{4.6}$. All baselines are evaluated using their default hyperparameters as reported in their respective papers. Unless otherwise specified, we use CLIP with a ViT-B/16 backbone as the primary VLM. To mitigate the effect of randomness, we report the average top-1 accuracy over three random seeds for all main experiments.
\subsection{Concept Generation Cost Analysis}
In our experiments, we utilize three variants of the GPT-4.1 model, involving GPT-4.1-Turbo, GPT-4.1-Mini, and GPT-4.1-Nano. The token pricing for each model is as follows:
\begin{figure*}[htbp]
    \centering
    \includegraphics[width=\linewidth]{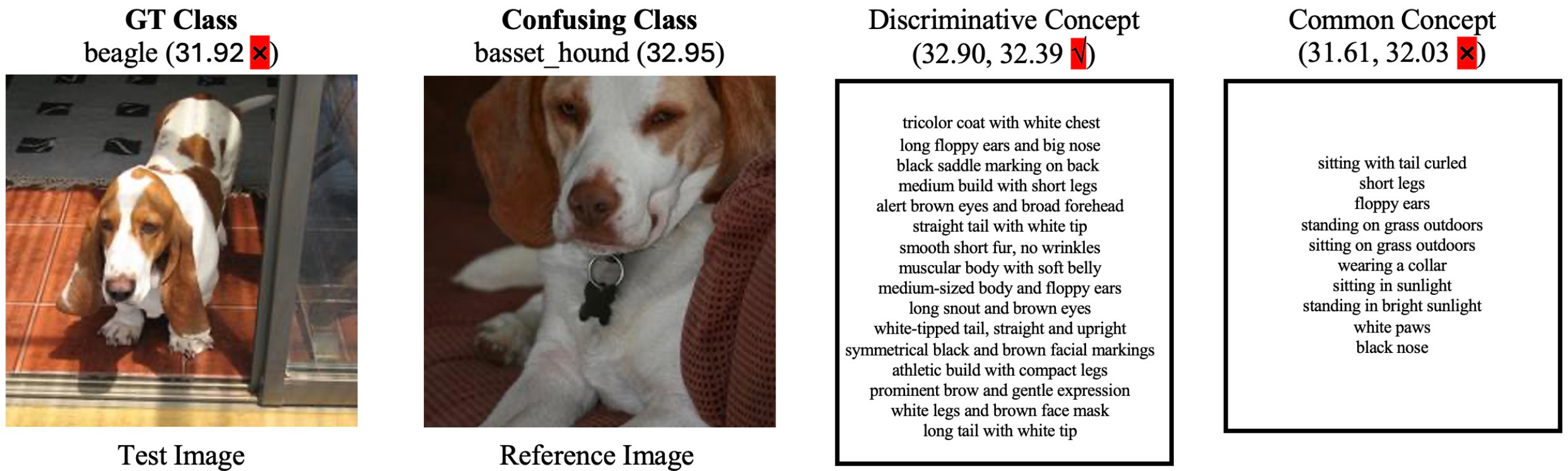}
    \caption{Illustrative examples of Discriminability in constructing the Concept Proposal Distribution.
We multiply the raw similarity values by CLIP's scaling factor.}
    \label{fig:q1}  
\end{figure*}
\begin{itemize}
    \item \textbf{GPT-4.1-Turbo}: \$2.00 per 1M input tokens, \$8.00 per 1M output tokens.
    \item \textbf{GPT-4.1-Mini}: \$0.40 per 1M input tokens, \$1.60 per 1M output tokens.
    \item \textbf{GPT-4.1-Nano}: \$0.10 per 1M input tokens, \$0.40 per 1M output tokens.
\end{itemize}
\begin{figure}[htbp]
    \centering
    \includegraphics[width=\linewidth]{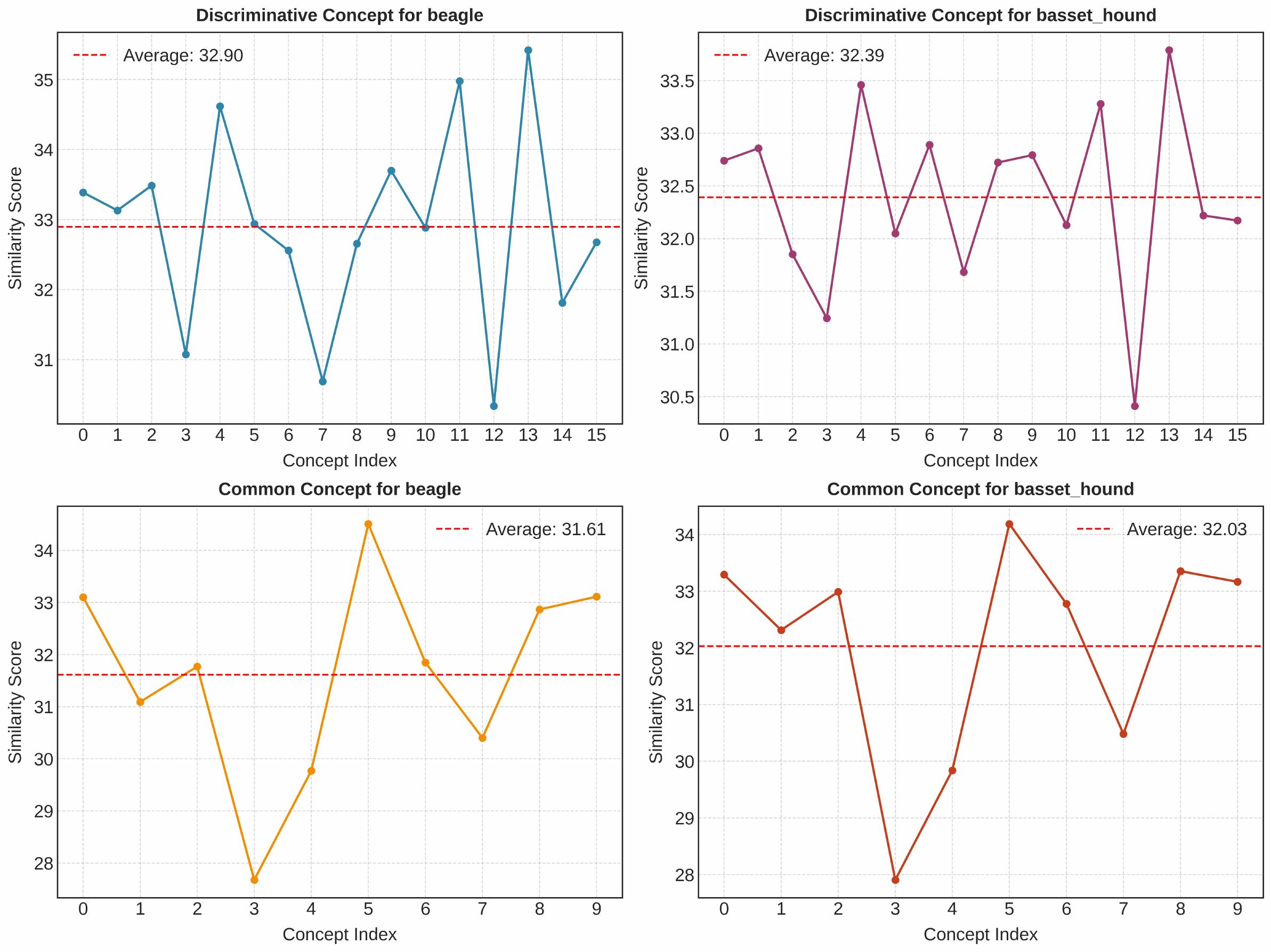}
    \caption{Visualization of similarity between prompts enhanced by discriminative and common concepts for the beagle and basset hound classes. We multiply the raw similarity values by CLIP's scaling factor.}
    \label{fig:simplot}  
\end{figure}
For each API call, we generate 10 candidate concepts. To obtain 50 unique concepts per class, considering possible duplication, we typically require approximately 6 API calls per class. Assuming a total of 1,000 classes, the approximate cost for each model variant is as follows:

\begin{itemize}
    \item \textbf{GPT-4.1-Turbo}: \$10.00
    \item \textbf{GPT-4.1-Mini}: \$2.00
    \item \textbf{GPT-4.1-Nano}: \$0.50
\end{itemize}
\begin{table*}[ht]
    \centering
    \caption{Prompt for discriminative concept generation. The red text represents variables that are dynamically replaced based on the specific datasets and image classes being tested. ``Core class" and ``Other Classes" correspond to $Y_i$ and $L_{i,j}$, respectively.}
    \begin{tabular}{>{\raggedright\arraybackslash}p{0.95\textwidth}}
        \hline
        \textbf{SYSTEM PROMPT:} \\
        You are a visual concept proposer tasked with enhancing text descriptions for zero-shot image classification on the test dataset using CLIP. \\

        Given: \\
        - A core class from the test dataset \\
        - The set of other classes in the dataset \\

        Task: \\
        Propose concise, visually discriminative concepts to append to the text description (i.e., ``A photo of \{core class\} with \{your concept\}'') that help CLIP better distinguish the core class from the other classes. \\

        Guidelines: \\
        - Analyze the unique visual characteristics of the core class compared to other classes\\
        - Propose concepts that capture these discriminative visual features. \\
        - Ensure concepts are concrete, easily understandable by CLIP, and specific to the test dataset. \\
        - Each concept should enable CLIP to more accurately classify images of the core class while minimizing confusion with other classes. \\

        \textbf{IMPORTANT: Your response must follow this exact format:} \\

        \texttt{<concepts begin>} \\
        concept1 \\
        concept2 \\
        concept3 \\
        \texttt{</concepts end>} \\

        \textbf{Rules:} \\
        - Start with \texttt{<concepts begin>} and end with \texttt{</concepts end>} \\
        - Each concept should be on a new line \\
        - Each concept MUST start with \texttt{"The final concept is: "} \\
        - Ensure concepts are clear, specific, and relevant to the core class \\
        - Avoid generic or ambiguous concepts \\
        - Each concept should be unique and distinct from others \\
        - Keep each concept brief (ideally $\leq$6 words), specific, and easy for CLIP to parse. \\
        \hline
        \textbf{USER PROMPT:} \\ 
        Core class: \textcolor{red}{[Core class]}. Other classes: \textcolor{red}{[Other Classes]}. Please generate 10 unique and visually discriminative concepts. Follow the required format and rules. \\
        \hline
    \end{tabular}
    \label{prompt:congen}
\end{table*}
\begin{table*}[ht]
    \centering
    \caption{Prompt for descriptive concept generation. The red text represents variables that will be replaced based on the specific datasets and image classes being tested. ``Core class" will be instantiated by $Y_i$.}
    \begin{tabular}{>{\raggedright\arraybackslash}p{0.95\textwidth}}
        \hline
        \textbf{SYSTEM PROMPT:} \\
        You are a visual concept proposer tasked with enhancing text descriptions for zero-shot image classification on the test dataset using CLIP.  \\
        
        Given: \\
        - A class from the test dataset \\

        Task: \\
        Propose descriptive concepts to append to the text description (i.e., ``A photo of \{core class\} with \{your concept\}'') that help CLIP better understand and recognize the core class. \\

        Guidelines: \\
        - Focus on the visual characteristics and attributes of the core class itself. \\
        - Generate descriptive concepts that capture various aspects, appearances, or contexts of the core class. \\
        - Ensure concepts are concrete, easily understandable by CLIP, and specific to the test dataset. \\
        - Think about different visual perspectives, settings, or attributes that describe the core class. \\

        \textbf{IMPORTANT: Your response must follow this exact format:} \\

        \texttt{<concepts begin>} \\
        concept1 \\
        concept2 \\
        concept3 \\
        \texttt{</concepts end>} \\

        \textbf{Rules:} \\
        - Start with \texttt{<concepts begin>} and end with \texttt{</concepts end>} \\
        - Each concept should be on a new line \\
        - Each concept MUST start with \texttt{"The final concept is: "} \\
        - Ensure concepts are clear, specific, and relevant to the given class \\
        - Avoid generic or ambiguous concepts \\
        - Each concept should be unique and distinct from others \\
        - Keep each concept brief (ideally $\leq$6 words), specific, and easy for CLIP to parse. \\
        \hline
        \textbf{USER PROMPT:} \\ 
        Core class: \textcolor{red}{[Core class]}. Please generate 10 unique and descriptive concepts that capture different visual aspects of this class.\\
        \hline
    \end{tabular}
    \label{prompt:desgen}
\end{table*}
\section{Supplementary Experiments}
\begin{table*}[t!]
\caption{Performance of zero-shot methods on different variants of ImageNet datasets. We report the average performance (Avg.) and average out-of-distribution performance (OOD Avg.). Best and second-best results are \textbf{bolded} and \underline{underlined}, respectively. The \textit{Auxiliary} column shows the number of views and prompts.}
\label{tab:oodperformance}
\centering
\begin{adjustbox}{width=0.9\linewidth, center}
\begin{tabular}{lccccc|cc|c}
\toprule
\textbf{Method} & \textbf{A} & \textbf{R} & \textbf{K} & \textbf{V} & \textbf{I} & \textbf{Avg.} & \textbf{OOD Avg.} & \textbf{Auxiliary} \\
\midrule
\textit{CLIP} & 47.80 & 73.99 & 46.15 & 60.84 & 66.73 & 59.10 & 57.20 & (1, 1) \\
CLIP + E. & 49.95 & \textbf{77.71} & 48.26 & 61.91 & 68.38 & 61.24 & 59.46 & (1, 80) \\
TPT & 54.63 & 77.04 & 47.97 & 63.41 & 68.94 & 62.40 & 60.76 & (64, 1) \\
MTA & 57.41 & 76.92 & 48.58 & 63.61 & 69.29 & 63.16 & 61.63 & (64, 1) \\
Cupl & 47.52 & 74.14 & 46.43 & 59.6 & 66.75 & 58.89 & 56.92 & (1, 16) \\
Cupl & 48.12 & 74.66 & 46.96 & 60.43 & 66.47 & 59.33 & 57.54 & (1, 50) \\
CGBC Prior & 49.44 & 75.30 & 48.06 & 61.56 & 68.21 & 60.51 & 58.59 & (1, 16) \\
CGBC & 49.54 & 75.35 & 48.16 & 62.11 & 69.22 & 60.88 & 58.79 & (1, 16) \\
\rowcolor{lavender1}
CGBC + View & \underline{61.68} & \underline{77.65} & \underline{49.24} & \underline{64.04} & \underline{70.5} & \underline{64.62} & \underline{63.15} & (64, 16) \\
CGBC Prior & 49.91 & 75.4 & 48.14 & 61.61 & 68.32 & 60.68 & 58.77 & (1, 50) \\
CGBC & 49.79 & 75.49 & 48.36 & 62.15 & 69.43 & 61.04 & 58.95 & (1, 50) \\
\rowcolor{lavender1}
CGBC + View & \textbf{62.19} & 77.41 & \textbf{49.45} & \textbf{64.37} & \textbf{70.64} & \textbf{64.81} & \textbf{63.36} & (64, 50) \\
\bottomrule
\end{tabular}
\end{adjustbox}
\end{table*}
\subsection{Robustness Analysis}  To evaluate the robustness of our proposed framework under natural distribution shifts, we compare its performance against other zero-shot methods, as presented in Table~\ref{tab:oodperformance}. The results indicate that CGBC consistently outperforms CGBC Prior, highlighting the necessity of likelihood-based refinement for the concept prior distribution. Although both CGBC and CGBC Prior maintain a clear performance margin over CLIP, they fall behind view-based methods under significant domain shifts. This observation suggests that multi-view information is essential in scenarios with substantial distributional changes or noise. To address this, we propose an enhanced variant, CGBC + View. Given an input image $X$ and its $N_{\text{aug}}$ augmented views, we apply the CGBC framework to each view to obtain robust similarity scores. We then compute the entropy of the predicted distributions for each view, select the top 10\% of views with the lowest entropy (i.e., most confident predictions), following~\citep{shu2022tpt} and average their results.  Empirical results show that CGBC + View outperforms view-augmentation-based baselines by an average of 2 percents across four OOD datasets when the number of prompts is set to 50. Notably, it achieves particularly strong performance on ImageNet-A. These phenomenon further validates the effectiveness of concept-based prompt augmentation in enhancing generalization under distribution shifts.

\subsection{Running Time Analysis} 
All experiments were conducted using four NVIDIA 3090Ti GPUs. To ensure a fair comparison of computational efficiency, we report the inference-time adaptation cost for TPT, MTA, CGBC, and CGBC + View on the ImageNet dataset under identical hardware conditions. Specifically, we exclude model loading time and assume that text prompt embeddings are precomputed. The total adaptation times are approximately 11 hours 26 minutes for TPT, 3 hours 42 minutes for MTA, 2 minutes 43 seconds for CGBC, and 2 hours 33 minutes for CGBC + View. These results demonstrate that our proposed CGBC framework requires significantly less adaptation time compared to view-augmentation-based zero-shot approaches, while still achieving superior performance. Furthermore, even when incorporating multi-view inputs, CGBC + View remains more efficient than standard view-augmented zero-shot methods, owing to its optimization-free nature.
\begin{figure}[htbp]
    \centering
    \includegraphics[width=\linewidth]{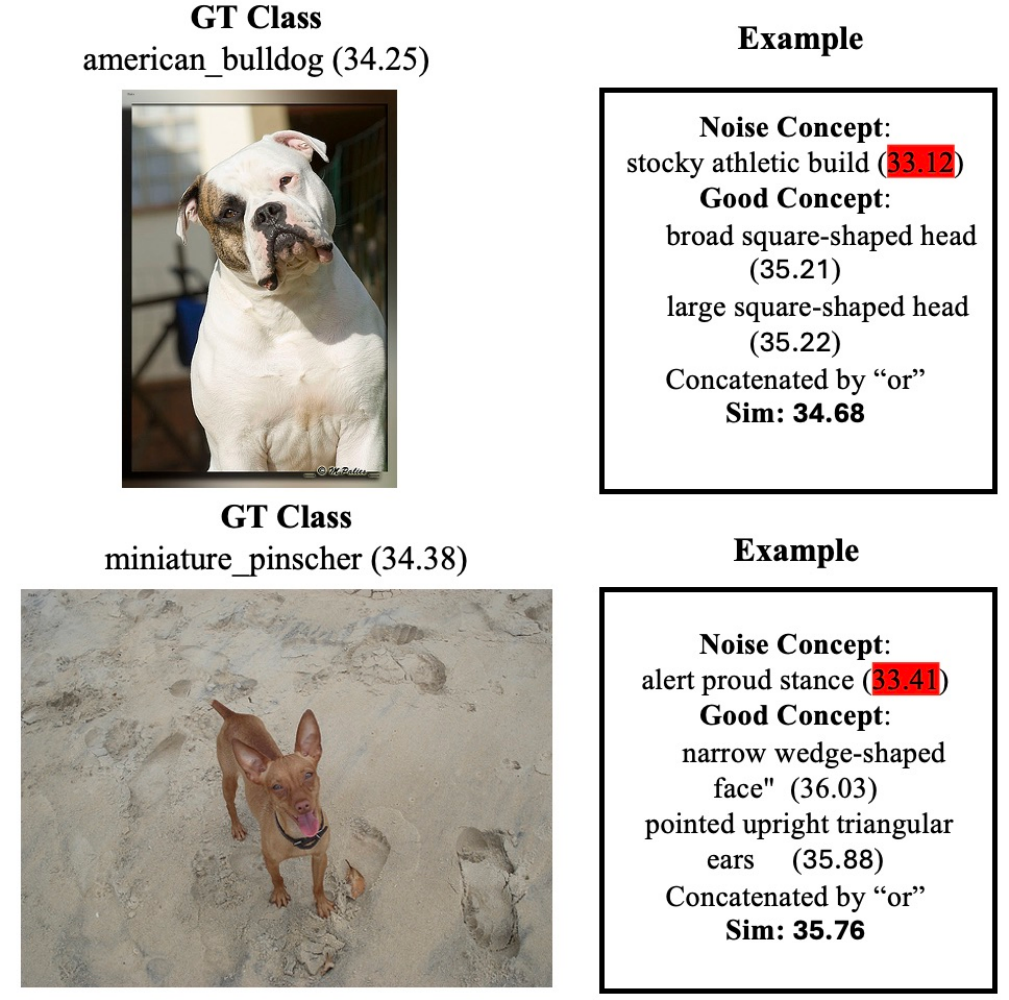}
    \caption{Illustrative examples of Compositionality in constructing the Concept Proposal Distribution.
We multiply the raw similarity values by CLIP's scaling factor.}
    \label{fig:q2}  
\end{figure}
\subsection{Qualitative Analysis}
In this subsection, we use illustrative examples to demonstrate the roles of discriminability and compositionality in constructing the concept proposal distribution. 
\textbf{Discriminability} Compared to descriptive prompts, our proposed LLM-driven Concept Synthesis Pipeline is more likely to generate discriminative concepts through the use of discriminative prompts. This suggests that while descriptive prompts can also produce effective discriminative prompts, the number and quality are generally lower than ours. The effectiveness of our approach is particularly evident when comparing discriminative concepts with shared concepts that are likely to be generated by descriptive prompts and are common to both classes.
As shown in Fig. \ref{fig:q1}, for the class beagle, when using CLIP with the original prompt format "a photo of {class}" instantiated with beagle and basset hound, the model misclassifies the image. Even when using a shared concept, a concept frequently generated by descriptive prompts and common to both classes, the reduction in the decision margin between the two classes is insufficient to correct the misclassification. In contrast, when using a discriminative concept, we can effectively adjust the margin and correct the classification. Fig.~\ref{fig:simplot} further illustrates how prompts enhanced with different types of concepts (common vs. discriminative) affect the similarity between the test image and both class prompts.
\textbf{Compositionality}: Among the atomic concepts generated by the LLM, there often exist noisy or outlier concepts as analyzed in the main body. In addition to the likelihood  function that helps mitigate the impact of such outliers, as illustrated in Fig. \ref{fig:q2}, compositional concepts can also alleviate the effect of poor atomic concepts by combining them with more informative ones. However, when noisy concepts constitute a significant portion of the atomic concept set, compositionality alone becomes insufficient. In such cases, a likelihood-based refinement of the concept prior remains essential to effectively suppress the influence of outlier concepts. \textbf{Diversity}: For a more comprehensive assessment of diversity, we refer readers to the quantitative analyses reported in the main text and to the final concept set released in our open-source repository.
\subsection{Supplementary Ablation studies}
\begin{table}[H]
\caption{The rows marked '(avg)' indicate results using concept prompt averaging to replace ``or''. We use 50 prompts).}
\label{tab:table10}
\centering
\begin{adjustbox}{width=\linewidth, center}
\begin{tabular}{lcccccccccccc}
\toprule
\textbf{Method} & \textbf{SUN.} & \textbf{Air.} & \textbf{Eur.} & \textbf{Car.} & \textbf{Food.} & \textbf{Pets} & \textbf{Flow.} & \textbf{Cal.} & \textbf{DTD} & \textbf{UCF} & \textbf{Ima.} & \textbf{Avg} \\
\midrule
CGBC Prior       & 68.5 & 26.1 & 59.3 & 66.6 & 83.8 & 88.5 & 72.4 & 94.0 & 55.1 & 71.8 & 68.3 & \underline{68.6} \\
Prior (avg) & 68.5 & 26.1 & 55.6 & 66.9 & 84.2 & 87.9 & 71.7 & 94.2 & 52.6 & 70.2 & 68.9 & 67.9 \\
CGBC             & 68.8 & 26.0 & 60.3 & 66.6 & 83.9 & 90.7 & 73.7 & 94.4 & 55.7 & 72.3 & 69.4 & \textbf{69.3} \\
CGBC (avg)       & 68.7 & 26.1 & 56.6 & 67.0 & 84.3 & 88.5 & 72.3 & 94.2 & 52.9 & 70.6 & 69.1 & 68.2 \\
\bottomrule
\end{tabular}
\end{adjustbox}
\end{table}
\noindent\textbf{Reasonability of using ``or'' for concept composition.}
We compose atomic concepts using natural-language coordination with ``or'' (e.g., ``\texttt{A or B}''). Although CLIP does not implement symbolic logic, it is trained on large-scale image--text data and can leverage the distributional semantics of common coordination patterns; thus ``or'' serves as a cue for \emph{alternatives}, encouraging a text representation compatible with multiple visual realizations. We additionally evaluate a non-compositional baseline that averages the CLIP text embeddings of the atomic prompts within a concept (Table~\ref{tab:table10}, \textit{Prior (avg)}). Empirically, ``or'' outperforms prompt averaging in our setting. We attribute this to the fact that textual composition preserves syntactic structure and a learned coordination pattern, whereas embedding averaging is a purely mixture that can dilute distinctive semantics.

\noindent\textbf{Versatility across concept generators.}
Beyond the GPT-family concept generators validated in the main text, we further evaluate CGBC using the Gemini series as the concept generator. Table~\ref{tab:model-benchmark} shows that CGBC maintains consistent performance across model groups: replacing GPT-4.1 turbo with Gemini 2.5 Flash, Gemini 2.5 Flash Lite, or Gemini 2.0 Flash results in only marginal changes in accuracy on all datasets, with the overall average remaining within a narrow band (61.0 vs.\ 59.9--60.4). These results indicate that CGBC is not tied to a specific LLM family and is robust to the choice of concept generator, supporting its practicality for deployment under different cost/latency constraints.

\begin{table}[t]
\centering
\caption{Performance on various model group.}
\label{tab:model-benchmark}
\begin{adjustbox}{width=\linewidth, center}
\begin{tabular}{lrrrrrrrr}
\toprule
\textbf{Avg} & \textbf{\(|L_i|\)} & \textbf{EuroSAT} & \textbf{DTD} & \textbf{Aircraft} & \textbf{Pets} & \textbf{UCF101} & \textbf{Average} & \textbf{Avg/EuroSAT} \\
\midrule
GPT-4.1 turbo      & 10 & 60.3 & 55.7 & 26.0 & 90.7 & 72.3 & 61.0 & 61.2 \\
Gemini 2.5 Flash   & 10 & 59.5 & 54.3 & 26.0 & 90.6 & 71.3 & 60.4 & 60.6 \\
Gemini 2.5 Flash Lite & 10 & 58.9 & 53.2 & 26.4 & 89.8 & 71.4 & 59.9 & 60.1 \\
Gemini 2.0 Flash   & 10 & 58.8 & 53.0 & 26.1 & 90.6 & 71.5 & 60.0 & 60.2 \\
\bottomrule
\end{tabular}
\end{adjustbox}
\end{table}

\section{Theoretical Analysis of Adaptive Soft-Trim based likelihood}
\subsection{Proof of Robust Guarantee}
We restate Theorem~\ref{thm:robust_mean} below for convenience.

\begin{theorem*}[Robust Guarantee]
If the size of concept set $\mathcal{C}_i$ for each class $Y_i$ satisfies $M_i \geq \frac{C_0 \log(1/\delta)}{\rho_i^2}$
for a universal constant $C_0$, then with probability at least $1-\delta$, the estimation error of $\hat{\mu}_{i}$ is bounded by:
\begin{equation}
\small
|\hat{\mu}_{i} - \mu_i| \leq C_1\sigma_i\rho_i + C_2\sigma_i\sqrt{\frac{\log(1/\delta)}{M_i}} + \frac{C_3\sigma_i}{k}
\end{equation}
where $C_1$, $C_2$, and $C_3$ are universal constants and $k$ is sigmoid slope parameter in Eq. \eqref{wij}.
\end{theorem*}
\begin{proof}
The proof proceeds by decomposing the total estimation error into three components: (1) a bias term arising from the adversarial contamination, (2) a statistical error term due to finite sampling, and (3) an approximation error term from the soft-trimming mechanism.

Let the set of samples be $\mathcal{S}_i = \{S_{i,1}, \dots, S_{i,M_i}\}$. By the Huber contamination model, we can conceptually partition $\mathcal{S}_i$ into a set of inliers $\mathcal{G}_i$ drawn from the 1-sub-Gaussian distribution $\mathcal{N}(\mu_i, \sigma_i^2)$ and a set of outliers $\mathcal{O}_i$ from an arbitrary adversarial distribution. We have $|\mathcal{G}_i| \geq (1-\rho_i)M_i$ and $|\mathcal{O}_i| \leq \rho_i M_i$.

The Adaptive Soft-Trim estimator is defined as $\hat{\mu}_{i} = \frac{\sum_{j=1}^{M_i} w_{i,j} S_{i,j}}{\sum_{j=1}^{M_i} w_{i,j}}$, where the weights $w_{i,j}$ are a decreasing function of the normalized deviation from the sample median, i.e., $w_{i,j} = f(-|S_{i,j} - m_i| / \text{MAD}_i)$ for some decreasing function $f$ (e.g., a  sigmoid we utlize in our implementation).

\paragraph{Robustness of Median and MAD}

A foundational result in robust statistics is that the median and Median Absolute Deviation (MAD) are robust estimators of location and scale. Under the given sample complexity condition  $M_i \geq \frac{C_0 \log(1/\delta)}{\rho_i^2}$, it can be shown  that with probability at least $1-\delta$~\cite{diakonikolas2018robustly}:
\begin{align}
    |m_i - \mu_i| &\leq O(\sigma_i \rho_i) \label{eq:median_robust} \\
    |\text{MAD}_i - c\sigma_i| &\leq O(\sigma_i \rho_i) \label{eq:mad_robust}
\end{align}
where $m_i = \text{median}(\mathcal{S}_i)$ and $c$ is the constant factor relating MAD to the standard deviation for the underlying clean distribution (e.g., $c \approx 0.6745$ for a Gaussian). This ensures that the center and scale used for trimming are themselves close to the true parameters, up to a factor of the contamination rate.
\paragraph{Error Decomposition}
We analyze the error $|\hat{\mu}_{i} - \mu_i|$ by triangle inequality. Let $W = \sum_{j=1}^{M_i} w_{i,j}$.
\begin{small}
\begin{align}
|\hat{\mu}_{i} - \mu_i| &= \frac{1}{W} \left| \sum_{j=1}^{M_i} w_{i,j} (S_{i,j} - \mu_i) \right| \notag\\
&\leq \frac{1}{W} \left| \sum_{j \in \mathcal{G}_i} w_{i,j} (S_{i,j} - \mu_i) \right| 
+ \frac{1}{W} \left| \sum_{j \in \mathcal{O}_i} w_{i,j} (S_{i,j} - \mu_i) \right|
\end{align}
\end{small}

\noindent\textbf{Bias from Contamination}: The second term, $\frac{1}{W} |\sum_{j \in \mathcal{O}_i} w_{i,j} (S_{i,j} - \mu_i)|$, captures the influence of outliers. Our proposed adaptive soft-trim based likelihood is designed to down-weight points with large deviations $|S_{i,j} - m_i|$. While outliers can be arbitrary, their influence is controlled. The theory of robust mean estimation~\cite{huber1992robust,diakonikolas2018robustly} establishes that for any estimator in this class, the bias induced by a $\rho_i$-fraction of contamination is at least $\Omega(\sigma_i \rho_i)$. Our proposed likelihood, being a variant of a trimmed mean, achieves a near-optimal rate. The bias from both the residual influence of down-weighted outliers and the error in the trimming center ($|m_i - \mu_i|$) is bounded by $C_1 \sigma_i \rho_i$.

\noindent\textbf{Statistical Error}: The first term, $\frac{1}{W} |\sum_{j \in \mathcal{G}_i} w_{i,j} (S_{i,j} - \mu_i)|$, represents the statistical error from estimating a mean from a finite number of the weighted inlier concepts. It shrinks as the sample size M increases. For $j \in \mathcal{G}_i$, $S_{i,j} - \mu_i$ are i.i.d. 1-sub-Gaussian random variables with mean 0. The sum is a weighted sum of such variables. By a weighted version of Hoeffding's inequality, for any fixed weights, this sum is bounded by $O(\sigma_i \sqrt{\sum w_{i,j}^2 \log(1/\delta)})$. Since $w_{i,j} \in [0, 1]$ and $W \approx (1-\rho_i)M_i$, we have $\sum w_{i,j}^2 \leq W$. The normalized sum is thus bounded by $O(\sigma_i \sqrt{\log(1/\delta)/W}) \approx O(\sigma_i \sqrt{\log(1/\delta)/M_i})$. This gives the term $C_2 \sigma_i \sqrt{\frac{\log(1/\delta)}{M_i}}$.

\noindent\textbf{Approximation Error} The use of a "soft" sigmoid function with slope $k$ instead of a "hard" threshold introduces a third error term. The weight function $w_{i,j}$ transitions from 1 to 0 over a region whose width is proportional to $1/k$ in the normalized deviation space. Inliers that fall into this transition region are unduly down-weighted, and outliers that fall into it are not fully removed. This imperfect trimming introduces a bias. The magnitude of this error is proportional to the width of the transition region, leading to an error term of the form $C_3 \sigma_i / k$. As $k \to \infty$, the soft-trim approaches a hard-trim, and this term vanishes.

Combining the three error bounds via the triangle inequality, we obtain the final result. With probability at least $1-\delta$, the total error is bounded by the sum of the bias, statistical, and approximation errors:
\begin{equation}  
|\hat{\mu}_{i} - \mu_i| \leq C_1\sigma_i\rho_i + C_2\sigma_i\sqrt{\frac{\log(1/\delta)}{M_i}} + \frac{C_3\sigma_i}{k}
\end{equation}
This completes the proof.
\end{proof}
\subsection{Proof of Multi-class excess risk}
We restate Corollary~\ref{cor:excess_risk} below for convenience.
\begin{corollary*}[Multi-class excess risk]
Under the condition in Theorem~\ref{thm:robust_mean}, the classifier for image $X$ in Eq. \eqref{eq:1} satisfies:
\begin{equation}
\small
\mathcal{R} - \mathcal{R}^* \leq \Pr\left[\text{margin}_{\mu}(X) \leq 2 \cdot \max_i |\hat{\mu}_{i}(X) - \mu_i(X)|\right]
\end{equation}
where $\mathcal{R}$ is the risk of our classifier, $\mathcal{R}^*$ is the optimal Bayes risk, and $\text{margin}_{\mu}(X)$ is the margin of the true mean $\mu_i(X)$. Consequently, With probability at least $1 -\delta$, the excess risk is bounded by:
\begin{small}
\begin{align}
\mathcal{R} - \mathcal{R}^* 
&\leq \Pr\Big[\text{margin}_{\mu}(X) \leq 2\big(C_1\sigma_{\max}\rho_{\max} 
+ \\ &C_2\sigma_{\max}\sqrt{\tfrac{\log(K/\delta)}{M_{\min}}} \notag 
+ \tfrac{C_3\sigma_{\max}}{k}\big)\Big]
\end{align}
\end{small}where $K$ is the number of classes, $\sigma_{\max} = \max_i \sigma_i$, $\rho_{\max} = \max_i \rho_i$, and $M_{\min} = \min_i M_i$.
\end{corollary*}
\begin{proof}
The proof is in two parts. First, we relate the excess classification risk to the estimation error of the means $\mu_i(X)$. Second, we apply the bound from Theorem \ref{thm:robust_mean} to this relation.
\paragraph{Relating Excess Risk to Estimation Error}

Let $h^*(X) = \mathop{\arg\max}\limits_{i} \mu_i(X)$ be the Bayes optimal classifier for the true (uncontaminated) means. The excess risk of our classifier $\hat{h}(X) = \mathop{\arg\max}\limits_{i} \hat{\mu}_{i}(X)$ is upper-bounded by the probability of misclassification relative to the optimal one:
\begin{equation}
\mathcal{R}(\hat{h}) - \mathcal{R}^* \leq \mathbb{E}_X \left[ \mathbb{I}(\hat{h}(X) \neq h^*(X)) \right] = \Pr(\hat{h}(X) \neq h^*(X))
\end{equation}
A misclassification occurs if $\hat{h}(X) \neq h^*(X)$. Let $i^* = h^*(X)$ be the true optimal class. A misclassification implies that there exists some class $j \neq i^*$ such that $\hat{\mu}_{j}(X) > \hat{\mu}_{i^*}(X)$. We can establish a sufficient condition for this event. Consider the inequality:
\begin{small}
\begin{align*}
    \hat{\mu}_{j}(X) &> \hat{\mu}_{i^*}(X) \\
    \mu_j(X) + (\hat{\mu}_{j}(X) - \mu_j(X)) &> \mu_{i^*}(X) + (\hat{\mu}_{i^*}(X) - \mu_{i^*}(X)) \\
    \mu_{i^*}(X) - \mu_j(X) 
    &< (\hat{\mu}_{j}(X) - \mu_j(X)) \\
    &\quad - (\hat{\mu}_{i^*}(X) - \mu_{i^*}(X))
\end{align*}
\end{small}
By the triangle inequality, the right-hand side is bounded:
\begin{align*}
\mu_{i^*}(X) - \mu_j(X) &\leq |\hat{\mu}_{j}(X) - \mu_j(X)| + |\hat{\mu}_{i^*}(X) - \mu_{i^*}(X)| \\&\leq 2 \cdot \max_k |\hat{\mu}_{k}(X) - \mu_k(X)|
\end{align*}
The margin of the true means is defined as $\text{margin}_{\mu}(X) = \mu_{i^*}(X) - \max_{j \neq i^*} \mu_j(X)$. Since a misclassification implies the existence of a $j \neq i^*$ satisfying the above, it must be that the margin itself is bounded by this quantity. Therefore, the event $\{\hat{h}(X) \neq h^*(X)\}$ is a subset of the event $\{\text{margin}_{\mu}(X) \leq 2 \cdot \max_k |\hat{\mu}_{k}(X) - \mu_k(X)|\}$. This gives the first inequality of the corollary:
\begin{align*}
\mathcal{R}(\hat{h}) - \mathcal{R}^* \leq \Pr\left[\text{margin}_{\mu}(X) \leq 2 \cdot \max_i |\hat{\mu}_{i}(X) - \mu_i(X)|\right]
\end{align*}
\paragraph{Applying the High-Probability Bound}

Theorem \ref{thm:robust_mean} provides a bound on $|\hat{\mu}_{i}(X) - \mu_i(X)|$ for a single class $i$ that holds with probability $1-\delta$. We require a bound that holds simultaneously for all $K$ classes. Let $\delta' = \delta/K$. Applying Theorem \ref{thm:robust_mean} for each class $i \in \{1, \dots, K\}$ with failure probability $\delta'$, we have that with probability at least $1-\delta'$,
\begin{equation}
|\hat{\mu}_{i} - \mu_i| \leq C_1\sigma_i\rho_i + C_2\sigma_i\sqrt{\frac{\log(1/\delta')}{M_i}} + \frac{C_3\sigma_i}{k}
\end{equation}
By the union bound, the probability that at least one of these $K$ bounds fails is at most $\sum_{i=1}^K \delta' = K(\delta/K) = \delta$. Therefore, with probability at least $1-\delta$, the bound holds for all $i=1, \dots, K$ simultaneously.

On this high-probability event, we can bound the maximum estimation error:
\begin{footnotesize}
\begin{align*}
\max_i |\hat{\mu}_{i}(X) - \mu_i(X)| &\leq \max_i ( C_1\sigma_i\rho_i + C_2\sigma_i\sqrt{\frac{\log(K/\delta)}{M_i}} + \frac{C_3\sigma_i}{k}) \\
&\leq C_1\sigma_{\max}\rho_{\max} + C_2\sigma_{\max}\sqrt{\frac{\log(K/\delta)}{M_{\min}}} + \frac{C_3\sigma_{\max}}{k}
\end{align*}
\end{footnotesize}where $\sigma_{\max} = \max_i \sigma_i$, $\rho_{\max} = \max_i \rho_i$, and $M_{\min} = \min_i M_i$. Substituting this high-probability bound into the result from \textbf{Relating Excess Risk to Estimation Error}, we arrive at the final expression for the excess risk:
\begin{small}
\begin{align}
\mathcal{R}(\hat{h}) - \mathcal{R}^* 
&\leq \Pr\Big[\text{margin}_{\mu}(X) \leq 2\big(C_1\sigma_{\max}\rho_{\max} 
+ \\ &C_2\sigma_{\max}\sqrt{\tfrac{\log(K/\delta)}{M_{\min}}} \notag 
+ \tfrac{C_3\sigma_{\max}}{k}\big)\Big]
\end{align}
\end{small}
This bound connects the classifier's performance to the intrinsic difficulty of the problem, captured by the distribution of $\text{margin}_{\mu}(X)$, and the quality of our robust estimation procedure.
\end{proof}
\section{Limitations}
CGBC yields the largest gains on datasets where CLIP can reliably interpret LLM-generated concepts (e.g., EuroSAT, DTD), whereas on harder domains (e.g., Aircraft) improvements are smaller and some prior prompts can match or slightly outperform CGBC. We attribute this variation to two factors: (i) occasional outlier concepts produced by LLMs and (ii) dataset-dependent limitations of CLIP in mapping text concepts to visual evidence. When CLIP’s concept understanding is adequate, filtering and reweighting reduce the impact of outliers and CGBC improves over \textit{CGBC Prior} (e.g., EuroSAT: 59.3$\rightarrow$60.3; ImageNet: 68.3$\rightarrow$69.4). When CLIP’s zero-shot baseline is weak, CLIP’s representational limitations can dominate and CGBC provides limited benefit (e.g., Aircraft: 26.1$\rightarrow$26.0), where more descriptive prompts (e.g., CUPL) may be easier for CLIP to align with images.

From a theoretical standpoint, our method is motivated by a Bayesian view of zero-shot classification but relies on practical approximations. A fully Bayesian treatment would require marginalizing over the space of all concepts \(C\) to compute \(p(Y\!\mid\!X)\), which is computationally infeasible; we instead approximate this integral using a small set of high-quality concepts synthesized by an LLM pipeline. Moreover, the likelihood term \(p(X\!\mid\!C)\) is generative in nature, while CLIP is a discriminative model; consequently, we use CLIP image--text similarity as an empirical proxy for \(p(X\!\mid\!C)\). Developing a more principled probabilistic grounding of this likelihood approximation remains an open direction.


\end{document}